\documentclass[journal]{IEEEtran}
\usepackage{cite}
\usepackage{amsmath,amssymb,amsfonts}
\usepackage{algorithmic}
\usepackage{graphicx}
\usepackage{textcomp}
\usepackage{xcolor}
\usepackage{multirow}
\usepackage{amsthm} 
\usepackage{bm}
\usepackage{subfigure}
\usepackage{tcolorbox}
\usepackage{booktabs}
\usepackage{balance}
\usepackage{ulem}
\usepackage{url}
 \usepackage[justification=centering]{caption}

\ifCLASSINFOpdf
\else
\fi
\begin{document}
%
\title{Tackling Long-tailed Distribution Issue in Graph Neural Networks via Normalization}
%
%
%

\author{Langzhang Liang,~
        Zenglin Xu,~\IEEEmembership{Senior Member,~IEEE,}
        Zixing Song,~ Irwin King,~\IEEEmembership{Fellow,~IEEE},~Yuan Qi,~and~Jieping Ye, \IEEEmembership{Fellow,~IEEE}
\thanks{Langzhang Liang and Zenglin Xu are with the Department
of Computer Science and Technology, Harbin Institute of Technology, Shenzhen,
Guangdong, China (E-mail: $\left\{\text{lazylzliang,zenglin}\right\}$@gmail.com).}
\thanks{Zixing Song and Iwrin King are with the Department of Computer Science and Engineering, Shatin, Hong Kong (E-mail: $\left\{\text{zxsong, king}\right\}$@cse.cuhk.edu.hk).}
\thanks{Yuan Qi is with the Artificial Intelligence Innovation and Incubation Institute, Fudan University, Shanghai, China (E-mail: qiyuan@fudan.edu.cn).}
\thanks{Jieping Ye is with the Department of Electrical Engineering and Computer Science, University of Michigan, Ann Arbor, USA (E-mail:jieping@gmail.com).}
\thanks{Zenglin Xu is the corresponding author}

\thanks{This work has been submitted to the IEEE for possible publication. Copyright may be transferred without notice, after which this version may no longer be accessible.}}

%
%

\markboth{Journal of \LaTeX\ Class Files,~Vol.~14, No.~8, August~2015}%
{Shell \MakeLowercase{\textit{et al.}}: Tackling Long-tailed Distribution Issue in Graph Neural Networks via Normalization}
%

\newtheorem{conj}{Conjecture}
\newtheorem{prop}{Proposition}



\maketitle

\begin{abstract}
  Graph Neural Networks (GNNs) have attracted much attention due to their superior learning capability. Despite the successful applications of GNNs in many areas, their performance suffers heavily from the long-tailed node degree distribution. Most prior studies tackle this issue by devising sophisticated model architectures. In this paper, we aim to improve the performance of tail nodes (low-degree or hard-to-classify nodes) via a generic and light normalization method. In detail, we propose a novel normalization method for GNNs, termed as ResNorm, which \textbf{Res}hapes a long-tailed distribution into a normal-like distribution via \textbf{Norm}alization. The ResNorm includes two operators. Firstly, the \textit{scale} operator reshapes the distribution of the node-wise standard deviation (NStd) so as to improve the accuracy of tail nodes. Secondly, the analysis of the behavior of the standard shift indicates that the standard shift serves as a preconditioner on the weight matrix, increasing the risk of over-smoothing. To address this issue, we design a new \textit{shift} operator for ResNorm, which simulates the degree-specific parameter strategy in a low-cost manner. Extensive experiments on various node classification benchmark datasets have validated the effectiveness of ResNorm in improving the performance of tail nodes as well as the overall performance.
\end{abstract}


%
\IEEEpeerreviewmaketitle

\section{Introduction}
Graph Neural Networks (GNNs)~\cite{DBLP:conf/nips/HamiltonYL17,DBLP:conf/iclr/KipfW17,DBLP:conf/iclr/VelickovicCCRLB18} have achieved remarkable successes in many fields, such as recommender systems~\cite{ tang2013social,DBLP:conf/www/Fan0LHZTY19}, information retrieval~\cite{DBLP:conf/www/ZhangLLXSLXYW21}, and drug discovery~\cite{DBLP:conf/www/GuoZYHW0C21,DBLP:conf/icml/LuoYJ21}. Many GNN variants have been proposed including Graph Convolutional Networks (GCNs)~\cite{DBLP:conf/iclr/KipfW17}, 
Graph Attention Networks (GATs)~\cite{DBLP:conf/iclr/VelickovicCCRLB18}, GraphSAGE~\cite{DBLP:conf/nips/HamiltonYL17}, etc.

Massive research efforts on GNNs have been devoted to studying issues of slow convergence~\cite{DBLP:conf/iclr/XuHLJ19,DBLP:conf/icml/ChenZS18} and the over-smoothing issue~\cite{Liu20deepGNN}. However, the long-tailed issue of graphs~\cite{DBLP:conf/recsys/LeeL11,DBLP:journals/kbs/HamedaniK19,DBLP:conf/recsys/Shi13}, which means a small number of nodes (\textit{i}.\textit{e}., head nodes) possess a large number of neighbors while the remaining majority (\textit{i}.\textit{e}., tail nodes) have only a few neighbors~\cite{DBLP:conf/kdd/LiuN021},
 has not been well explored in GNN literature. 
  
For example, in an online social network where each node represents a user and the links model the relationship between users, a celebrity may have millions of followers while most users have only a few followers.
As a consequence, the tail nodes commonly have inferior classification performance~\cite{tang2020investigating}. This can not only impact the overall performance across all nodes but also introduce biases and undermine fairness towards the under-represented tail nodes during the training of a GNN on these graphs.

To alleviate this limitation, we aim to design a normalization method for GNNs to deal with the long-tailed distribution prevalent in real-world graphs.
Normalization~\cite{DBLP:conf/icml/IoffeS15,DBLP:journals/corr/BaKH16} is commonly adopted to boost the training of deep learning models, such as LayerNorm~\cite{DBLP:journals/corr/BaKH16} for natural language understanding models and BatchNorm~~\cite{DBLP:conf/icml/IoffeS15} for image processing models. 
Although some studies have explored the normalization technique for GCNs~\cite{DBLP:conf/cikm/ZhouDWLHXF21,DBLP:conf/nips/Zhou0LZCH20,DBLP:conf/icml/CaiLXHL021,DBLP:conf/iclr/RongHXH20}, they neglect the long-tailed characteristic of node degree distribution. Unfortunately, how to design a normalization method that fits the long-tailed distribution of graph data remains unclear.

\begin{figure*}[t]
    \centering
    \includegraphics[width=1.5\columnwidth]{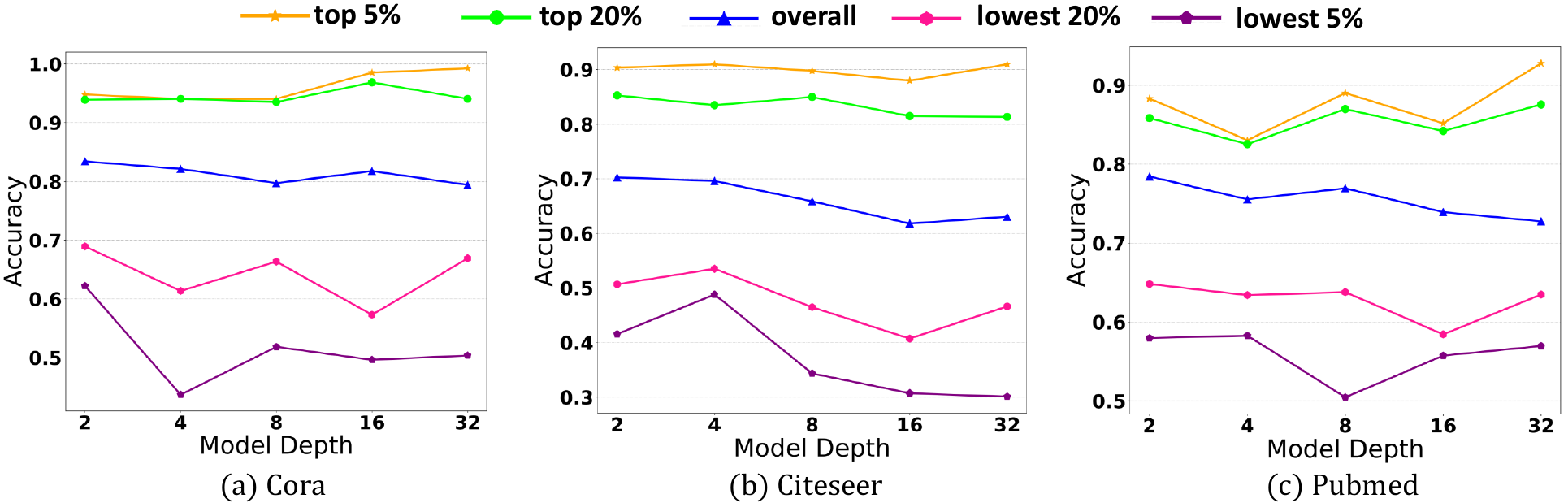}
    \caption{Classification accuracy statistics w.r.t. nodes with different Node-wise Standard Deviation (NStd) on the three datasets (\textit{i}.\textit{e}., Cora, Citeseer, and Pubmed). To sidestep the gradient vanishing problem, all GCNs are equipped with the residual connection. Each NStd value is computed based on the output of the corresponding previous hidden layer, respectively. }
    \label{fig:1}
\end{figure*}

As normalization is generally applied to node representations rather than the original input graph, we need to seek a proxy statistical measure instead of modifying the distribution of node degrees directly.
In detail, we consider the node-wise standard deviation (abbreviated as NStd) calculated from the node representations of the last layer of GCNs (as calculated in Eq.(\ref{eqn:nstd})). This choice is motivated by the similarity between the NStd distribution and its corresponding node degree distribution, as depicted in Figure \ref{fig:3}. 
To gain insights into the association between classification accuracy and NStd values, we categorize nodes into four groups: the top 5\%, the top 20\%, the lowest 5\%, and the lowest 20\% based on their NStd values. This analysis is conducted using three datasets (i.e., Cora, Citeseer, and Pubmed), and the classification accuracy is plotted against varying model depths in Figure~\ref{fig:1}. It is evident from the plot that nodes with higher NStd values are more likely to be correctly classified, while the tail nodes tend to be incorrectly classified.
Therefore, to improve the performance on the tail nodes, it is possible to reshape the NStd distribution of nodes such that the node degree distribution can be implicitly reshaped, \textit{i}.\textit{e}., uniforming the node degrees among head and tail nodes.


To this end, we propose a simple yet effective normalization method, named ResNorm, which includes two operators, i.e., the \textit{scale} operator and the \textit{shift} operator. In detail, the \textit{scale} operation of ResNorm reshapes the NStd distribution by reducing the gap between the head groups and the tail groups. We further provide a theoretical explanation for understanding how it works based on the influence distribution~\cite{DBLP:conf/icml/XuLTSKJ18}, as demonstrated in Figure~\ref{fig:influence_distribution}. For the design of the \textit{shift} operation of ResNorm, we mainly focus on the over-smoothing issue, which is a basic flaw of most popular GNN frameworks.  
Finally, we validate the effectiveness of ResNorm on nine popular node classification benchmarks. Experimental results show that ResNorm significantly improves the accuracy of the tail nodes, as well as the overall performance.

\section{Related Work}\label{sec:2}
In this section, we present related work on normalization techniques and long-tailed data classification. 

\subsection{Normalization Techniques for Graph Representation Learning}
Normalization is believed to reduce the training time substantially~\cite{DBLP:journals/corr/BaKH16}. \textit{Scale} and \textit{shift} are the main components of normalization. \textit{Scale} operation scales the given set of features down by its standard deviation. \textit{Shift} operation subtracts the mean from the given set of features. Formally, given a feature vector $\bm x \in \mathbb{R}^n$,  its mean denoted by $\mu=\frac{1}{n} \sum_{i=1}^n x_{i}$, and its variance denoted by $\sigma^2 = \frac{1}{n}\sum_{i=1}^n (x_i-\mu)^2$, a normalization step can be formulated as
\begin{equation}
    \text{Norm}(\bm x) = \gamma \frac{\bm x-\mu}{\sigma}+\beta,
\end{equation}
where $\gamma$ and $\beta$ are learnable parameters. Different normalization methods are applied to different feature vectors. BatchNorm~\cite{DBLP:conf/icml/IoffeS15} normalizes the features within the same channel across different samples in a batch, while LayerNorm~\cite{DBLP:journals/corr/BaKH16} normalizes the features within the same sample across different channels.

The successful applications of normalization in many domains have sparked research on benefiting the training of GNNs with normalization. GraphNorm~\cite{DBLP:conf/icml/CaiLXHL021} is adapted from InstanceNorm~\cite{DBLP:journals/corr/UlyanovVL16} and introduces a learnable parameter into the shift operation. NodeNorm removes the shift operation from normalization steps and scales a given node representation by its root of standard deviation. 
GraphNorm~\cite{DBLP:conf/icml/CaiLXHL021} aims to adaptively achieve a trade-off between preserving the information contained in the mean and accelerating the training process. NodeNorm~\cite{DBLP:conf/cikm/ZhouDWLHXF21} suggests that tackling the so-called variance inflammation issue via normalization. Besides, some normalization methods that tackle the over-smoothing issue have been proposed. For instance, PairNorm~\cite{DBLP:conf/iclr/ZhaoA20} mitigates over-smoothing by maintaining the node-pair distance, while Differentiable Group Normalization (DGN)~\cite{DBLP:conf/nips/Zhou0LZCH20} clusters nodes into multiple groups and then normalize them independently. Different from these methods, our proposed ResNorm addresses the problem of representation learning on long-tailed graphs. 

\subsection{Long-tailed Data Classification on Graphs}
Long-tailed data classification is an active topic in image classification~\cite{DBLP:conf/iccv/ZhangFWLQ17,DBLP:conf/cvpr/OuyangWZY16,DBLP:conf/nips/WangRH17} and is also known as the class imbalance problem that a small number of classes have a large number of samples while a large number of classes have only a few samples. 
However, long-tailed data classification is less studied in graph representation learning (GRL).
We emphasize that our definition of long-tailed data for graphs is different from that for images. Formally, a graph follows a long-tailed distribution if a small fraction of nodes have lots of neighbors while a large number of nodes have a few neighbors. Such long-tailed data pose a big challenge on GRL: it is hard to learn good representations for the tail nodes in downstream tasks due to their sparse edges and sensitivity to the bias. Most existing methods are designed for recommender system~\cite{DBLP:conf/www/SzpektorGM11,DBLP:conf/www/ZhangCYYHC21} or large-scale social networks~\cite{enders2008long,sankar2021graph,DBLP:conf/www/TraversoHTELP12}, where the long-tailed distribution issue is especially severe. Previous solutions utilize transfer learning to transfer knowledge from head items to tail items~\cite{DBLP:conf/www/ZhangCYYHC21,DBLP:conf/kdd/LiuN021}. Others suggest that adding different weights to user-item pairs~\cite{DBLP:journals/pvldb/YinCLYC12} and clustering~\cite{DBLP:conf/recsys/ParkT08}, as well as incorporating meta-learning~\cite{DBLP:conf/kdd/LeeIJCC19} and active learning~\cite{DBLP:journals/tkde/ZhuLHWGLC20}. Besides, SL-DSGCN~\cite{tang2020investigating} employs different parameters to train node representation with different degrees. Similarly, DEMO-Net~\cite{DBLP:conf/kdd/WuHX19} is another graph convolutional network that maps different structure subtrees to different feature vectors by using degree-specific parameters. We note that our normalization method is orthogonal to the existing methods as they mostly depend on fixed model architectures and lack flexibility, while the normalization layer can be easily incorporated into existing models in a plug-and-play manner.

\subsection{Graph Neural Networks}
In this section, we provide a brief review of three of the most representative GNN architectures.

Graph Convolutional Networks (GCNs)~\cite{DBLP:conf/iclr/KipfW17}, emerging as the most popular GNN architecture, have achieved high performance in the node classification task on various benchmarks~\cite{DBLP:conf/nips/HuFZDRLCL20}. GCNs perform a first-order approximated graph convolution in the graph spectral domain per layer. Mathematically, the layer of GCN is defined as
\begin{align}
    \bm H^{(l+1)} = \sigma(\tilde{\bm D}^{-\frac{1}{2}}\tilde{\bm A}\tilde{ \bm D}^{-\frac{1}{2}}\bm H^{(l)}\bm W^{(l)}).\label{eq:GNN_propagation}
\end{align}
Here, $\tilde{\bm D}^{-\frac{1}{2}}\tilde{\bm A}\tilde{\bm D}^{-\frac{1}{2}}$ is a propagation matrix with $\tilde{  \bm A}=\bm A+\bm I_n$, where $\bm I_n$ is an identity matrix and $\tilde{\bm D}_{ii} = \sum_j \tilde{\bm A}_{ij}$ is the degree matrix of $\tilde{\bm A}$, $\sigma(\cdot)$ denotes a non-linear function such as ReLU, $\bm{A}\in \mathbb{R}^{n \times n}$ is the adjacency matrix, $\bm W^{(l)}$ is a weight matrix at the $l$-th layer, $\bm H^{(l)}$ is the node representations matrix at the $l$-th layer.

Graph Attention Networks (GATs)~\cite{DBLP:conf/iclr/VelickovicCCRLB18} utilize the attention mechanism to adaptively learn representations by assigning learnable weights to neighbors. GATs update node representations by
    \begin{equation}
        \bm h_v^k = \sigma( \sum_{u \in \mathcal{N}(v) \cup v} \alpha_{vu}\bm W^k \bm h_u^{k-1}),
    \end{equation}
where $\mathcal{N}(v)$ denotes the neighbors of node $v$, and $\alpha_{vu}$ denotes an attention score indicating the importance of node $u$ to node $v$.
    
GraphSAGE~\cite{DBLP:conf/nips/HamiltonYL17} is a scalable GNN architecture based on neighbor sampling:
    \begin{equation}
    \begin{split}
        \bm h_{\mathcal{N}(v)}^k &= \text{AGGREGATE}(\left\{ \bm h_u^{k-1}|u \in \mathcal{N}(v) \right \})\\
        \bm h_{v}^k &= \sigma(\bm W^k \cdot \text{CONCAT}(\bm h_{\mathcal{N}(v)}^k, \bm h_{v}^{k-1})),
    \end{split}
    \end{equation}
where AGGREGATE denotes feature aggregation using mean-aggregator, max-aggregator, or LSTM-aggregator, $u \in \mathcal{N}(v)$ can be replaced with $u \in \mathcal{N}_k(v)$ which denotes the sampled neighbors of node $v$ at the k-th layer.

\section{Motivating Observations and Analysis}\label{sec:3}
In this section, we present observations on the NStd distributions followed by the interpretation. For a $d$-dimensional node representation, its NStd is computed across $d$ dimension features. Mathematically, for the node $v_i$ with $d$-dimensional representation $\bm x_i$, its NStd is calculated as
\begin{equation}
    \sigma_i = \sqrt{ \frac{1}{d} \sum_{j=1}^d (\bm x_{i,j} - \frac{1}{d} \sum_{j=1}^d \bm x_{i,j})^2}.
\label{eqn:nstd}
\end{equation}
\subsection{Observations}
We are interested in four types of nodes: a) nodes with the largest 5\% NStds; b) nodes with the largest 20\% NStds; c) nodes with the smallest 5\% NStds; d) nodes with the smallest 20\% NStds. Figure \ref{fig:1} reports the mean classification accuracy of these four types of nodes. Surprisingly, all of the results show the same pattern that those nodes with larger NStds consistently have higher classification accuracy. Such observations imply that NStd is a good indicator that has a strong positive correlation with node classification accuracy.

Since a larger NStd is generally associated with higher accuracy, we also refer to larger NStd nodes as head nodes and refer to smaller NStd nodes as tail nodes. The phenomenon that the gap of NStd existing between head and tail nodes is termed as variance expressing (variance expresses the difficulty level of node classification).

Figure \ref{fig:3} shows the NStd distribution and the node degree distribution of the Cora dataset. They are similar in shape and both exhibit the characteristics of the long-tailed distribution. Other datasets show the same pattern and the corresponding figures are provided in Appendix.
The similarity between these two distributions implies that reshaping the NStd distribution may achieve a similar effect of uniforming the node degree distribution. This leads to the core idea of this work: By reshaping the NStd distribution, we implicitly reshape the node degree distribution to some extent, which alleviates the issue caused by long-tailed distribution.

\begin{figure}[t]
    \centering
    \subfigure[Cora (NStd)] {\includegraphics[width=0.48\columnwidth]{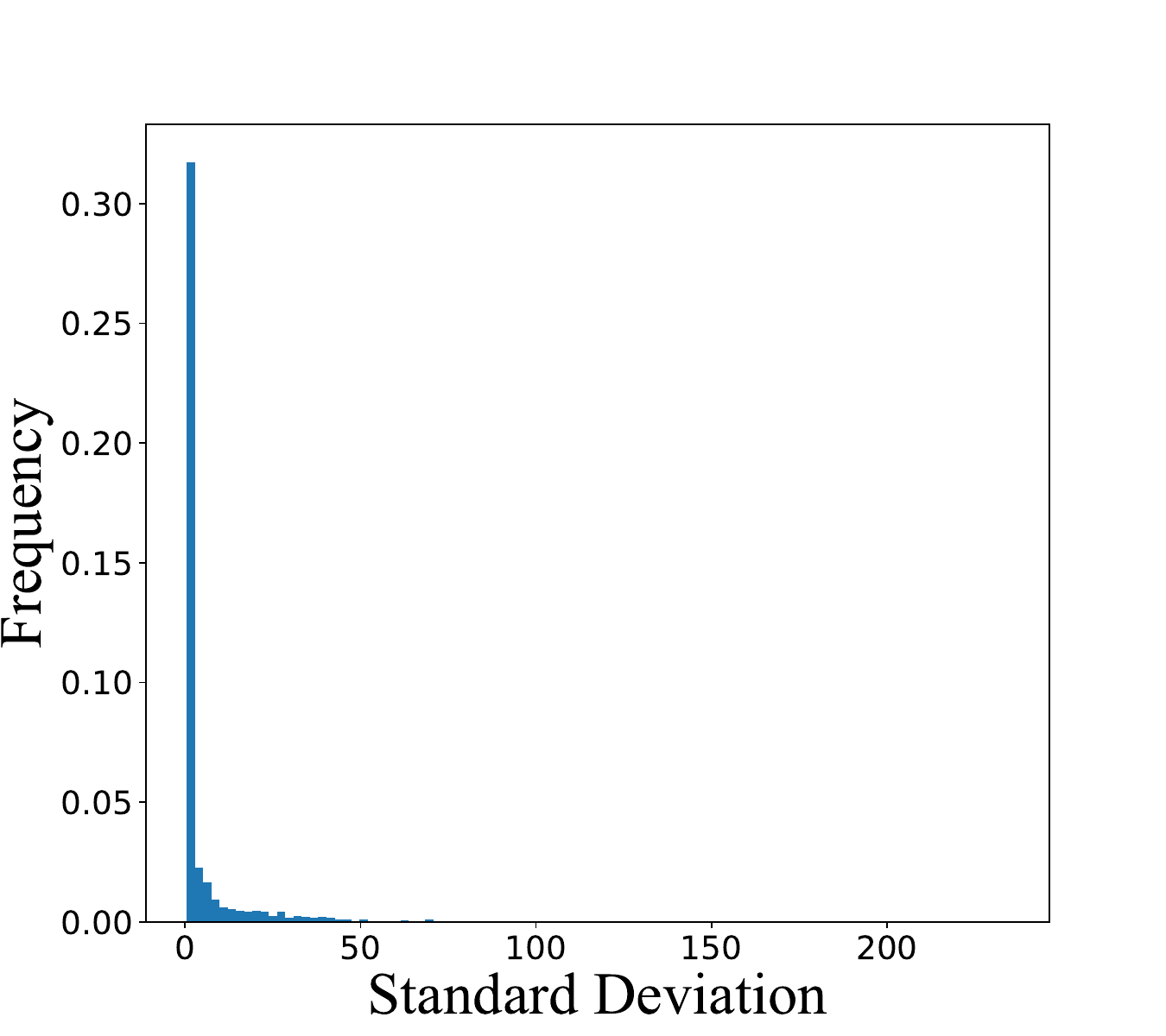}}
    \subfigure[Cora (Node degree)]
    {\includegraphics[width=0.48\columnwidth]{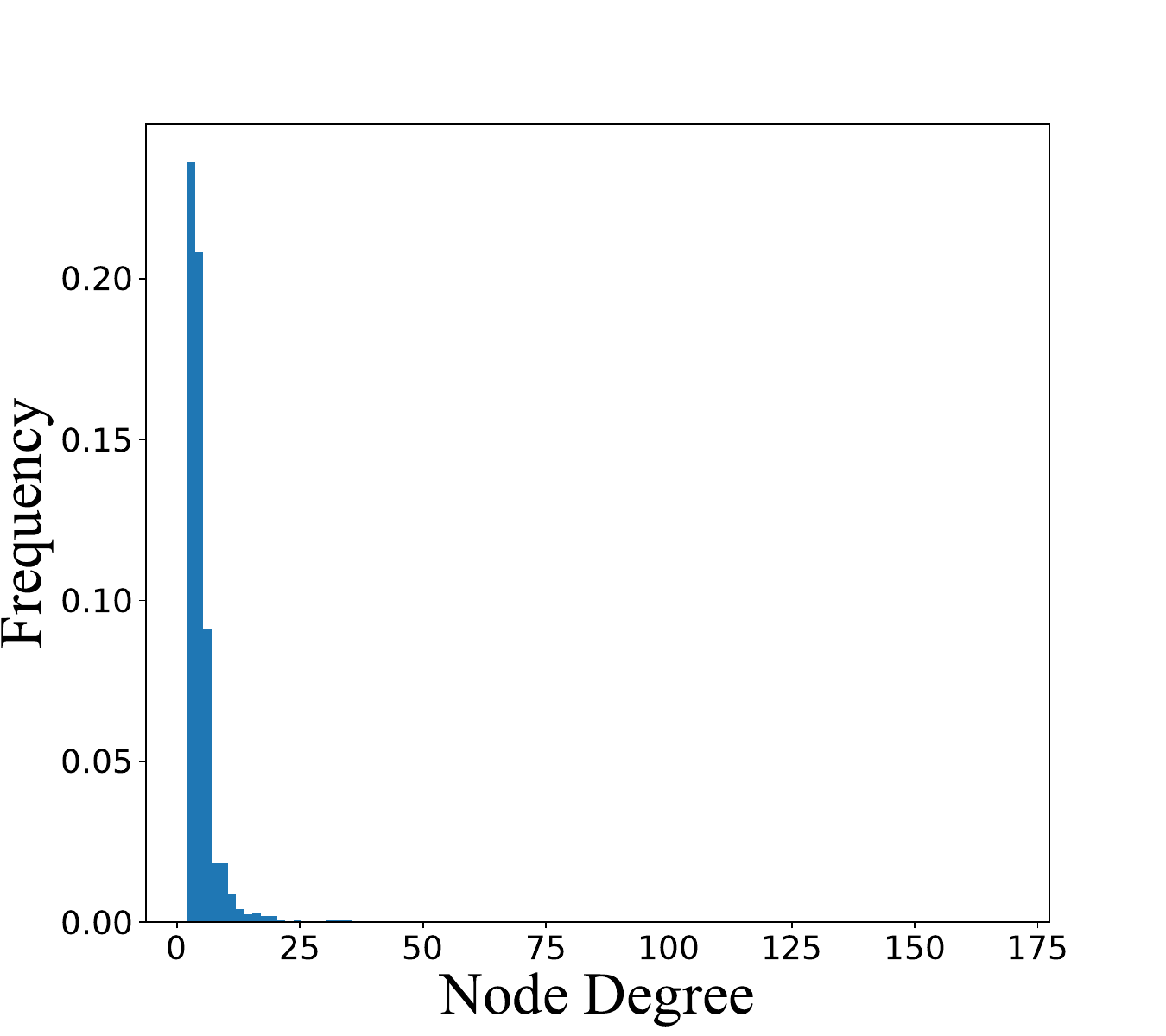}\label{fig:3b}}
    \caption{NStd distribution of the output of the last layer and the node degree distribution of the corresponding dataset. We use 32-layer GCNs equipped with the residual connection.}
    \label{fig:3}
\end{figure}

\subsection{Analysis on the Variance Expressing}
In this section, we investigate why each node's relative magnitude of NStd can indicate its difficulty level of node classification. We have the following speculation.

\begin{conj}\label{conj:1}
Nodes with higher ratios of intra-class neighbors or larger degrees tend to have larger NStd after multiple propagation operations.
\end{conj}
We verify this conjecture by using real-world datasets. To avoid the influence of transformation, we use de-parameterized 32-layer GCNs that only perform propagation at each layer. For the analysis of the asymptotic behavior of such linear GCNs, the interested reader is referred to~\cite{DBLP:conf/nips/WangWYL21}. Concretely, we perform 32 times propagation operations to the raw input features, then report the intra-class ratio and degree of different node sets in Table \ref{tab:1}. When the difference in intra-class ratio is relatively large, NStd monotonically increases as the intra-class ratio increases (Cora and Citeseer). On the other hand, NStd monotonically increases as node degree increases (Pubmed, because the mean node degree of Pubmed is significantly larger than that of the other two datasets.). 
\begin{table}[t]
\centering
\caption{Mean node degree and intra-class ratio (ICR) statistics for the different node sets. +$k$\% corresponds to the nodes with the largest $k\%$ Nstds after 32 times of de-parameterized propagation.}
\label{tab:1}
\begin{tabular}{c|c|ccc}
\midrule
\midrule
\multicolumn{2}{c|}{Metric}         & Cora & Citeseer & Pubmed \\
\hline
\multirow{4}{*}{ICR} &+1\% &\textbf{0.925}   & \textbf{1.0}    &0.783 \\   
\cline{2-5}
                     &+5\%   &0.892    &0.856   &0.804\\    
                     \cline{2-5}
                     &+30\%   &0.821    & 0.761  &0.803   \\
                     \cline{2-5}
                     &+80\%   &0.814   &0.745    &0.804\\

\hline
\multirow{4}{*}{Degree} &+1\% &7.44   &0.76   &\textbf{37.0}\\ 
\cline{2-5}
                        &+5\% &4.23   &1.51   &24.60\\  
                        \cline{2-5}
                        &+30\%    &6.40   &1.73   &10.87\\
                        \cline{2-5}
                        &+80\%    &4.50   &3.01   &5.36\\
\midrule
\end{tabular}
\end{table}
Such observation suggests that the propagation tends to make the NStd of nodes with larger intra-class ratios or degrees relatively larger than other nodes in a graph.
\begin{figure*}[t]
    \centering
    \subfigure[Raw] {\includegraphics[width=0.47\columnwidth]{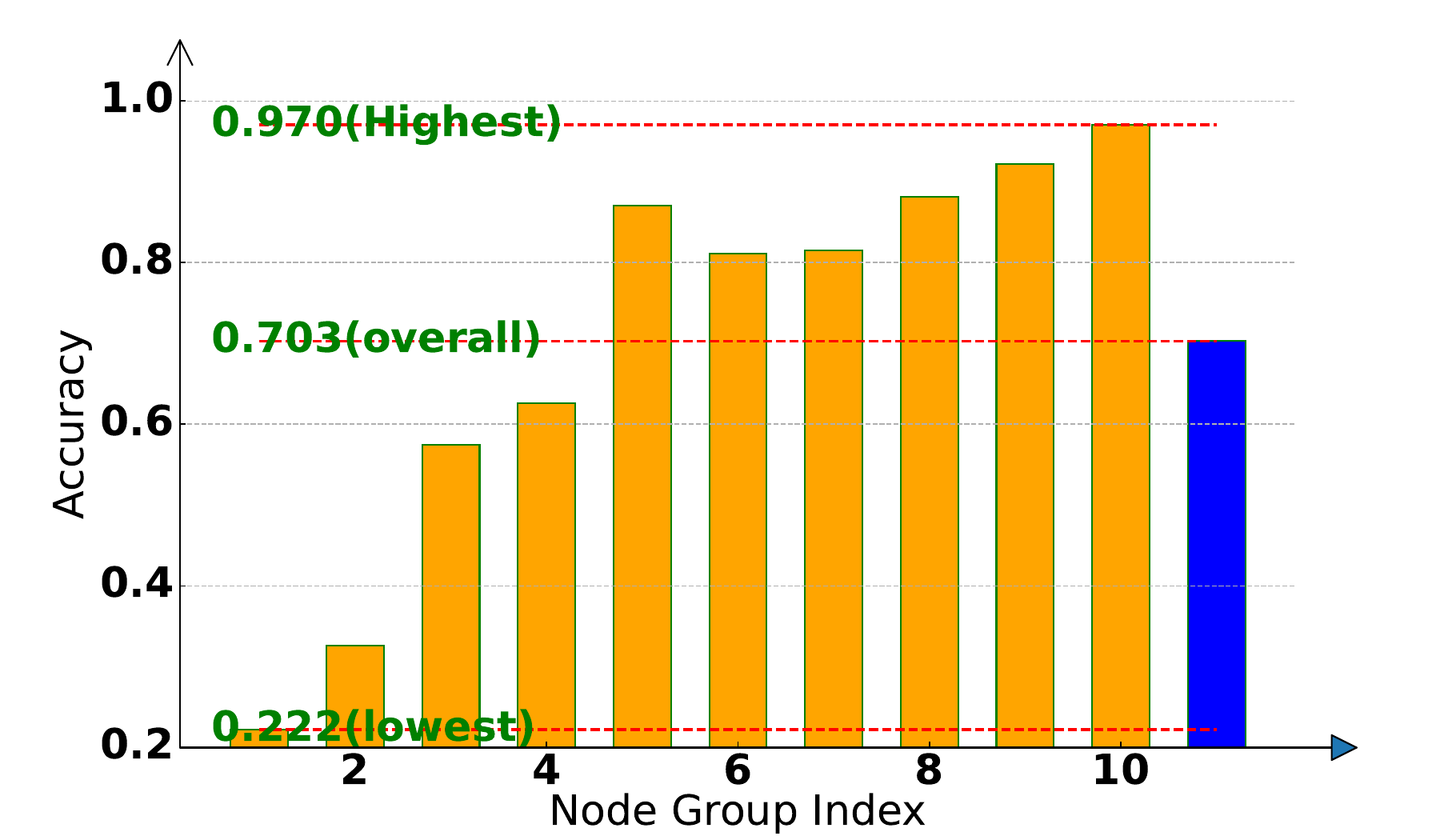}\label{fig5:a}}
    \subfigure[Unit Variance]
    {\includegraphics[width=0.47\columnwidth]{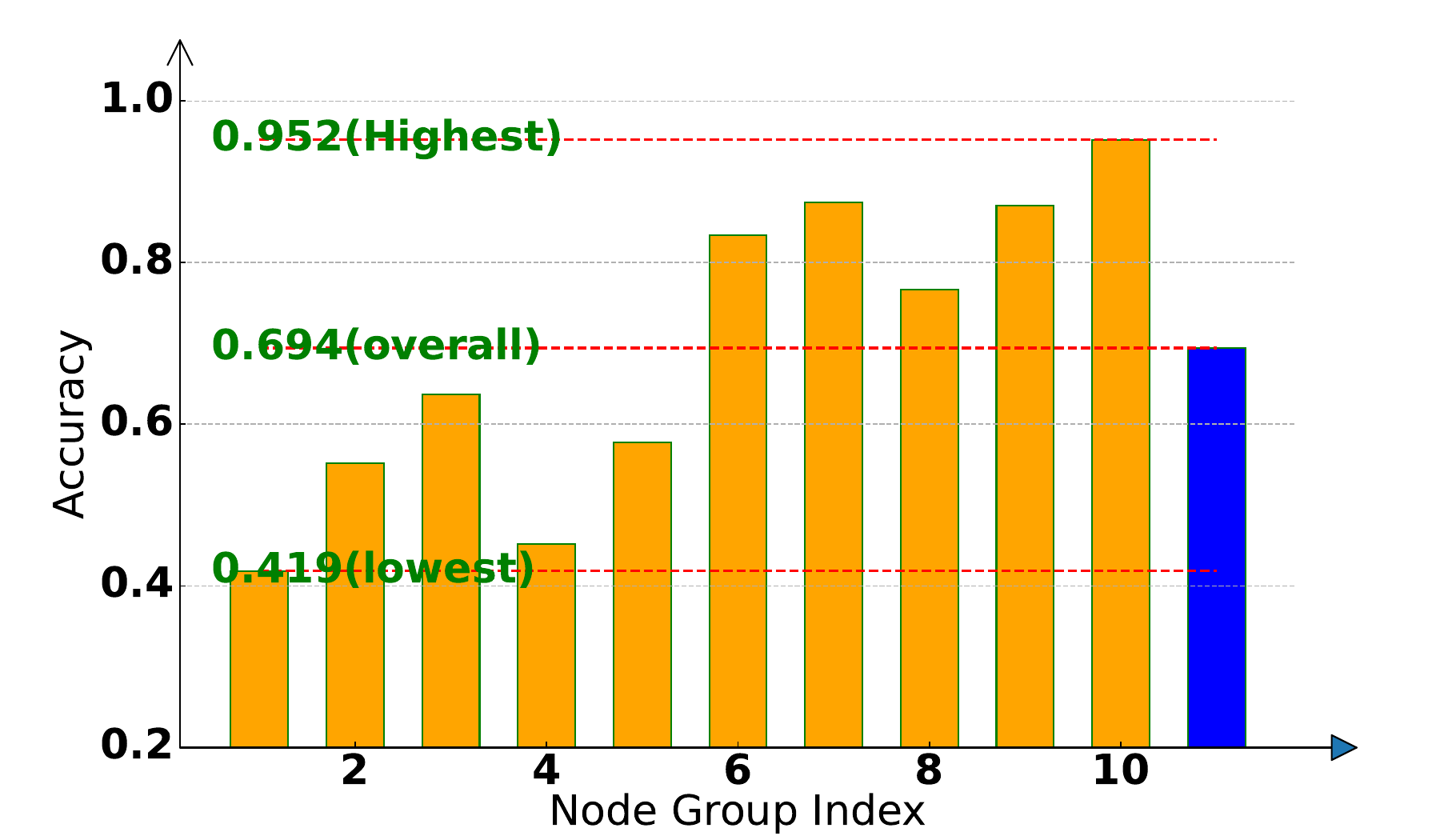}}
    \subfigure[Square Root]
    {\includegraphics[width=0.47\columnwidth]{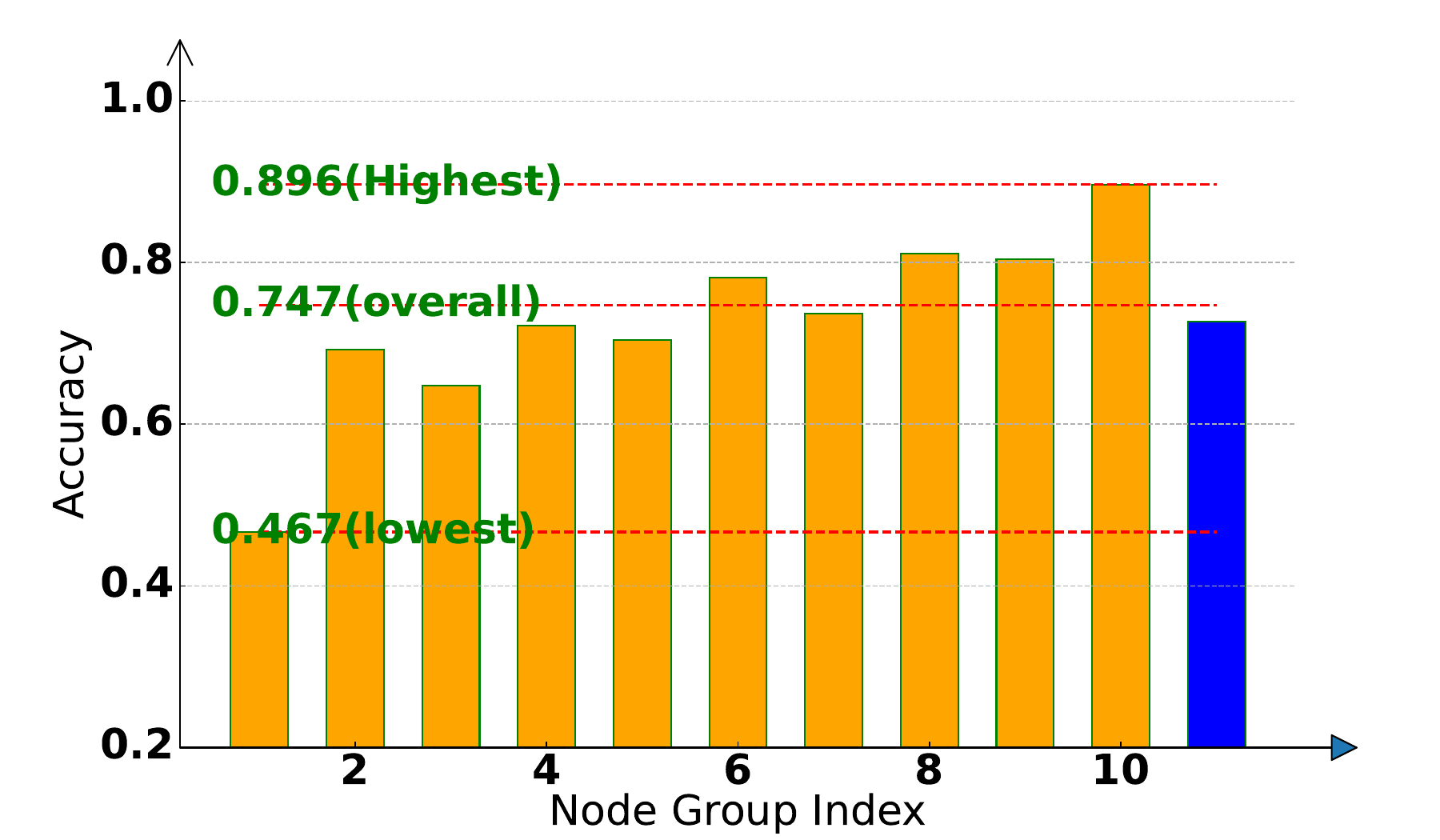}\label{fig5:c}}
    \subfigure[Reciprocal Square Root]
    {\includegraphics[width=0.47\columnwidth]{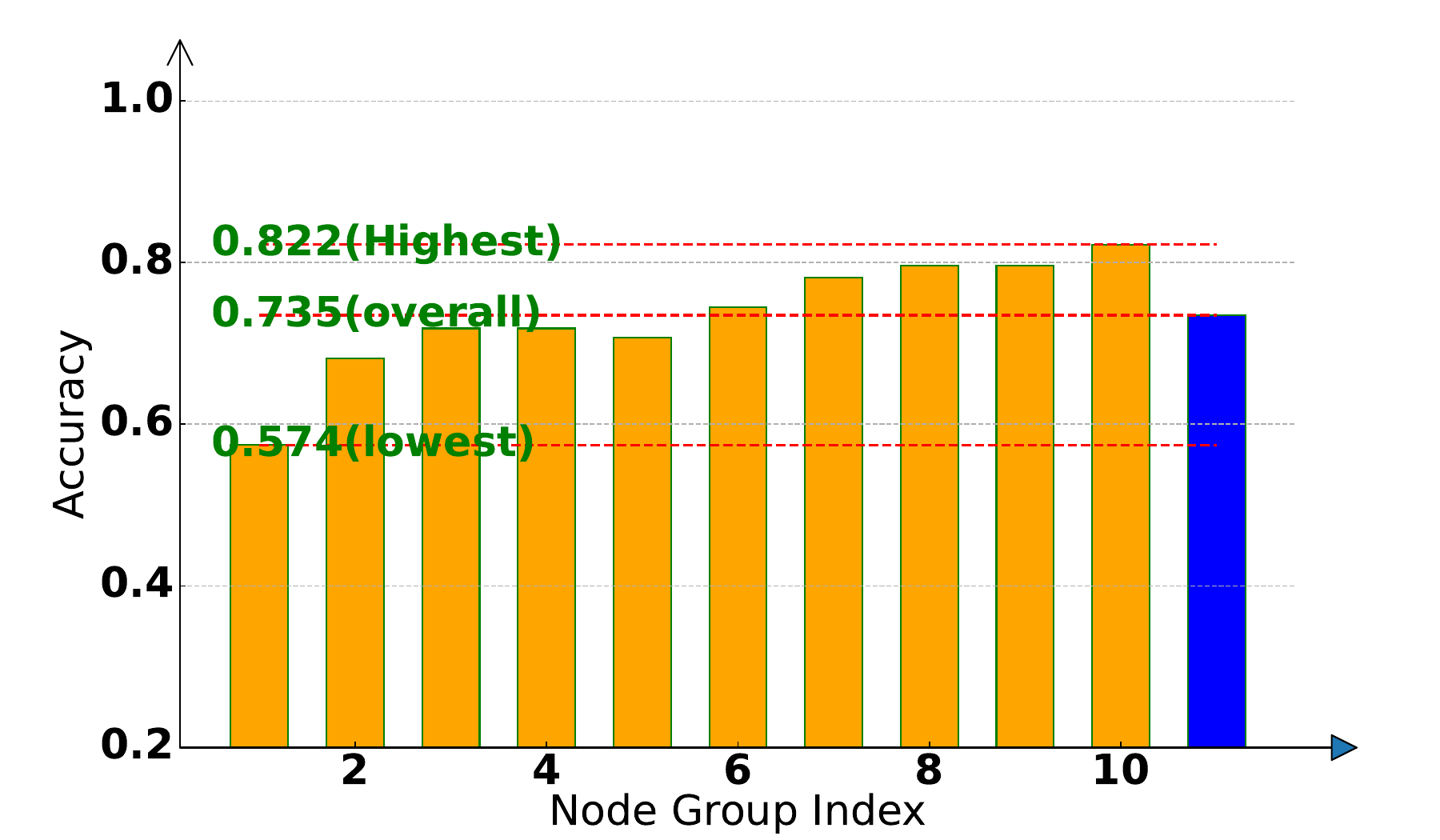}
    \label{fig5:d}}
    \caption{Accuracy of different node sets, where the indices of node set are ranked from smaller to larger NStd. For example, nodes in group 0 have the smallest 10\% NStds. Blue bars correspond to the overall accuracy.}
    \label{fig:5}
\end{figure*}

\subsection{Our Remedy: Distribution Reshaping}
Motivated by the similarity between the NStd distribution and the node degree distribution, and the positive correlation between NStd and classification accuracy, we propose to leverage the NStd distribution to resolve the long-tailed graph data classification problem. Since the NStd distribution is similar to the node degree distribution, reshaping the NStd distribution may implicitly reshape the node degree distribution, thus benefiting the classification accuracy of tail nodes. Our main idea is to improve the classification accuracy of tail nodes by reshaping the NStd distribution from long-tailed to normal-like. 

Power transformation is a classic technique to improve normality~\cite{tukey1957comparative}, typically including square root transformation, reciprocal (negative exponent) transformation, and identity transformation. Among them, the reciprocal transformation will reverse the order of the original data. We study the effects of several transformations from the perspectives of performance and normality. Specifically, we apply power transformation to node representations at each layer. For instance, assume $\bm x_i$ is the node representation of node $v_i$ and $\sigma_i$ is the NStd computed with $\bm x_i$, then performing square root transformation can be formulated as $\text{scale}(\bm x_i) = \frac{\bm x_i}{\sigma_i^{0.5}}$. We report the mean accuracy of node groups with different magnitudes of NStds in Figure \ref{fig:5}. We can observe the inferior performance of tail nodes when no transformation is applied (Figure \ref{fig5:a}). Both square root and reciprocal square root transformations can substantially improve the overall and tail performance. We may understand the mechanism of power transformation by observing the performance change of the head and tail. Both square root and reciprocal square root transformation increase the relative magnitudes of the NStds of the tails and reduce that of the heads, which gives rise to performance improvement and degradation for the tail and head, respectively. Thanks to the long-tailed distribution, although we lose some accuracy of the head, we still achieve overall performance improvement. Reciprocal transformation is the best to uniform the performance across heads and tails. However, square root transformation improves the overall performance the most. An intuitive explanation is that reciprocal transformation reverses the order of NStd across the head and tail nodes, making the NStds of tail nodes larger than that of head nodes. Hence, the models are forced to pay more attention to tail nodes and reduce the bias toward head nodes. However, such severe transformation leads to huge instability. We can see the huge performance drop of head nodes in Figure~\ref{fig5:d}. Therefore, square root transformation is more suitable. This inspires us to introduce positive power transformation as a component of the scale operation.

\subsection{An Interpretation from the Perspective of Influence Distribution}
We here investigate the effect of the node-wise scale from the perspective of influence distribution. For the discussion of the effect of scale in optimization, see, e.g., ~\cite{DBLP:conf/icml/IoffeS15}. Since the scale is applied node-wise, the magnitudes of node representations are scaled differently and individually. Hence the influence distribution is modified. The concept of influence function is initially introduced in~\cite{DBLP:conf/icml/KohL17}, then employed to analyze the influence distribution of each node by~\cite{DBLP:conf/icml/XuLTSKJ18}. For any node $x$, influence distribution captures the relative influences of all other nodes. For completeness, we also give the definition of influence distribution.
\newtheorem{definition}{Definition}
\begin{definition}
The influence score $I(x,y)$ that measures the influence of node $y$ on $x$ is the sum of the absolute values of the entries of the Jacobian matrix $[\frac{\partial h_x^{(k)}}{\partial h_y^{(0)}}]$, where $h_x^{(k)}$ is the learned node representation of node $x$ at the $k$-th (last) layer and $h_y^{(0)}$ is the input feature of node $y$. The influence distribution $I_x(y)$ of $x$ is calculated as $I_x(y) = \frac{I(x,y)}{\sum_z I(x,z)}$.
\end{definition}
In expectation, the influence distribution of $x$ is equivalent to the k-step random walk distribution with $x$ as the starting point~\cite{DBLP:conf/icml/XuLTSKJ18}, which is \textit{static} and only correlated to the graph topology. Here, \textit{static} means that the expected influence distribution remains unchanged during the training process. The lack of adaptability and the irrelevance with node features harm the performance as the similarity between two nodes is determined by structure and features. The idea is that the relative influence of any node $y$ on any node $x$ should be learnable and also exploit the node features, which we achieve by incorporating the node-wise scale.
\begin{prop}\label{theorem:1}
The node-wise scale endows the influence distribution with dynamic and adaptability which tends to improve the influence of smaller-NStd nodes.
\end{prop}
This provides an interpretation using the influence distribution. The node-wise scale penalizes the large-NStd nodes by reducing their impact on other nodes, leading to less impact on the training of the whole model.
We visualize examples in Figure \ref{fig:influence_distribution}. The influence distribution of GCN only depends on the graph structure, \textit{i}.\textit{e}., the relative influence is inversely proportional to the distance between two nodes. However, when the scale is applied, the influence of the two nodes in the upper left (have degrees 1 and 2, respectively) is increased. This observation is consistent with the above analysis.
\begin{figure}[t]
    \centering
    \subfigure[GCN (random walk)] {\includegraphics[width=0.48\columnwidth]{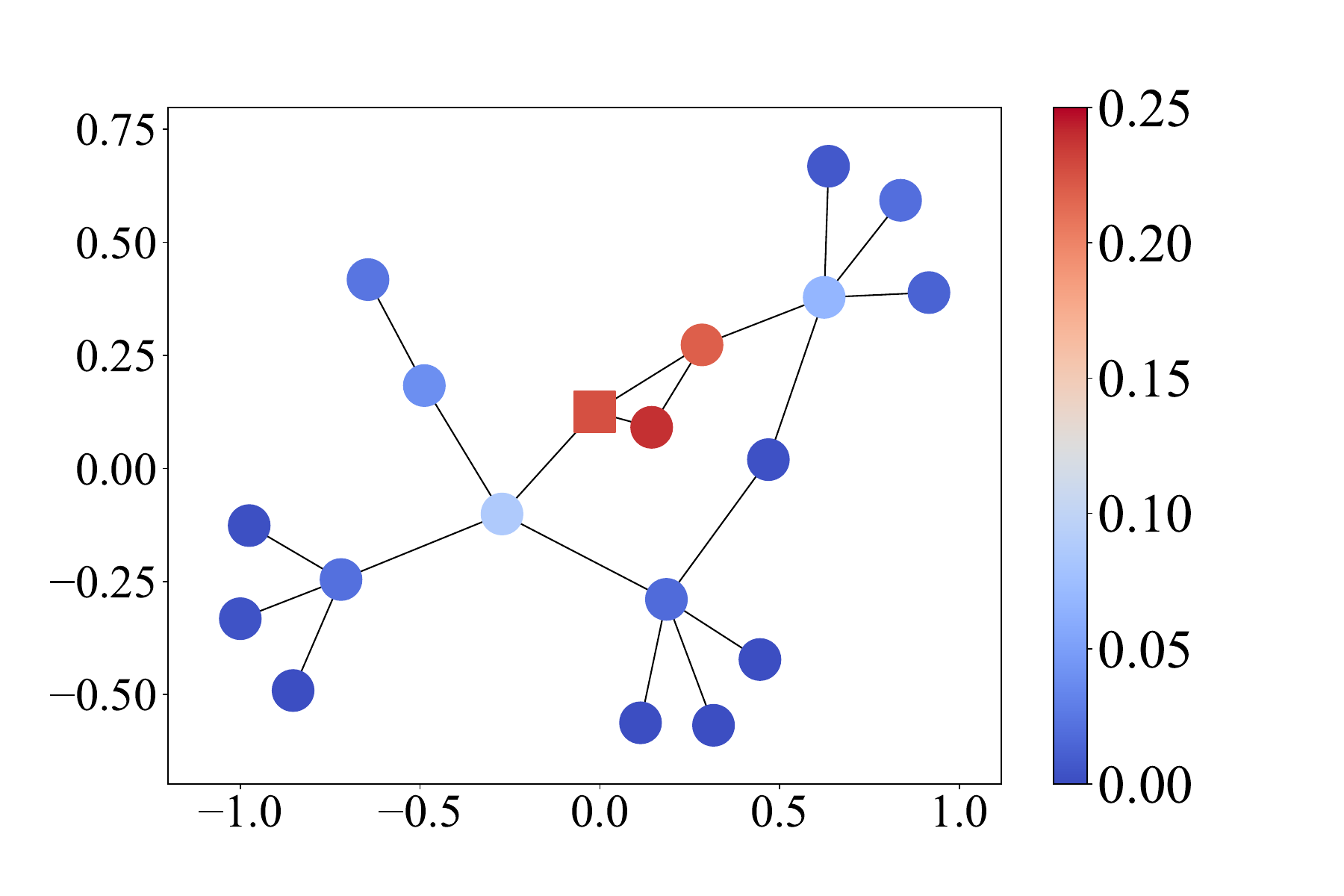}}
    \subfigure[GCN+Scale]
    {\includegraphics[width=0.48\columnwidth]{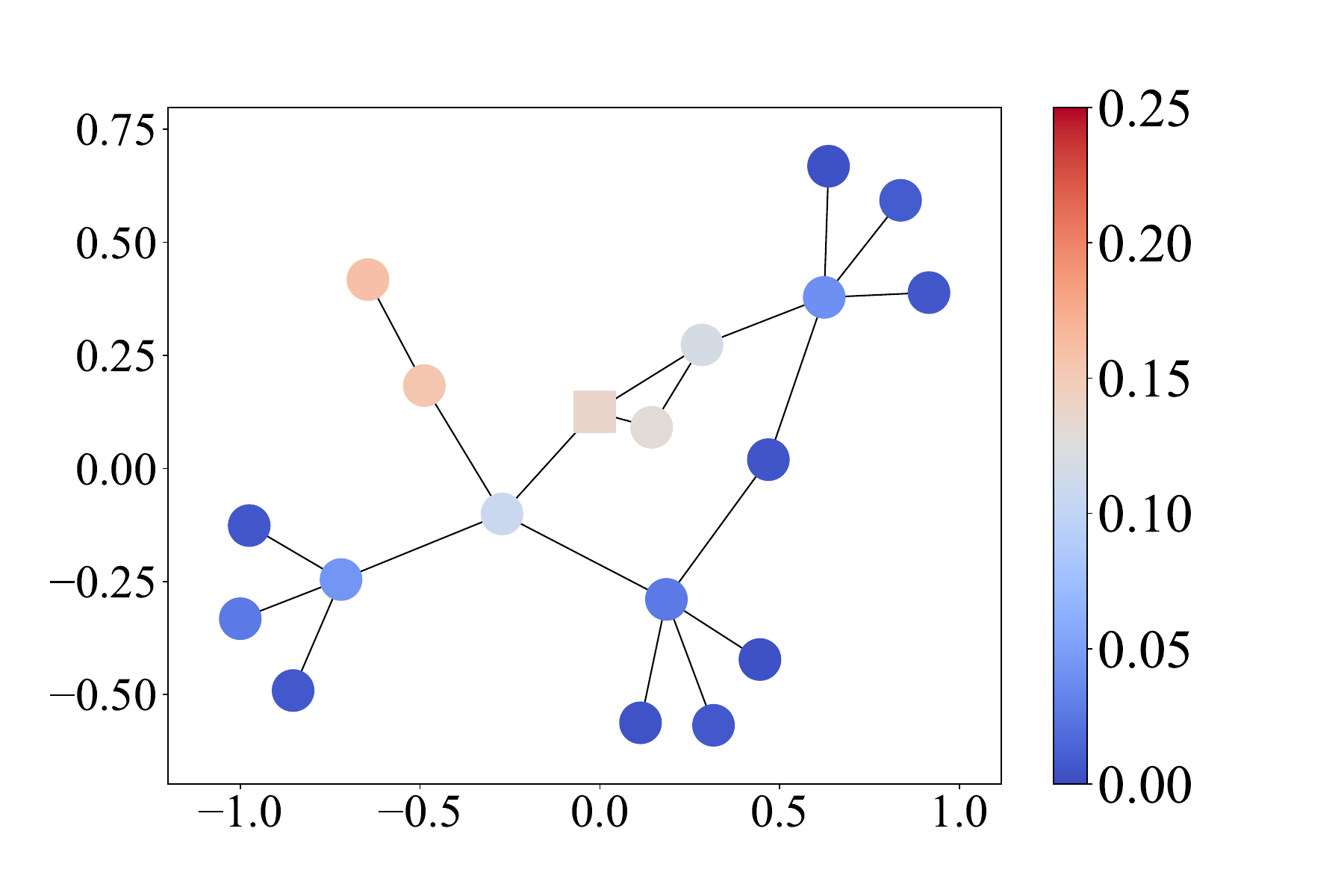}}
    \caption{Influence distribution of a target node (the 5th node of Cora dataset). We only consider its neighbors within 3 hops. The target node is labeled square.}
    \label{fig:influence_distribution}
\end{figure}


\section{The proposed ResNorm}\label{sec:5}

\subsection{Distribution-Reshaping Scale}
In this part, we formulate the scale operation of ResNorm. We first include power transformation, \textit{i}.\textit{e}., $\bm x_i'=\frac{\bm{x_i}}{\sigma_i^e}$, where $0<e<1$ is a hyper-parameter. Note that after power transformation, the residual NStd for node $v_i$ is $\sigma_i^{1-e}$. We take the mean of the partial NStd to further transform the distribution into more normal-like by $\frac{\bm x_i'\cdot \overline{\sigma^\gamma}}{\sigma_i^\gamma}$, where $\overline{\sigma^\gamma}$ is the mean of all nodes' $\sigma^\gamma$ and $0<\gamma<1$ a hyper-parameter. With this setting, $e$ controls to what extent we normalize the magnitudes of NStd, while $\gamma$ controls the extent of reshaping the distribution of NStd into normal-like. This separation makes the scale flexible in determining the contributions of these two main effects.
Formally, the scale operation of ResNorm is defined as
\begin{equation}
    \text{Scale}(\bm x_i) = \frac{\bm x_i \cdot \overline{\sigma^{\gamma}}}{\sigma_i ^ {e+\gamma}}.
\end{equation}

After scaling, the NStd of node $v_i$ becomes $\sigma_i^{(1-e-\gamma)}\cdot \overline{\sigma^\gamma}$, and the magnitude is similar to $\sigma_i^{(1-e)}$, which is partially normalized.
One of the most desirable effects of normalization is the ``scale-invariant'' property: the output remains unchanged even if the weights are scaled. As a result, normalization implicitly tunes the learning rate~\cite{DBLP:conf/icml/IoffeS15} and decouples the optimization of direction and length of the parameters~\cite{DBLP:conf/aistats/KohlerDLHZN19}.

\begin{prop}\label{prop:2}
Let $\bm W$ be a weight matrix in a linear layer, $\bm W'=\delta \bm W$ the scaled weight matrix where $\delta>0$, $\bm{x_i}$ a column feature vector, $f(\cdot)$ is an arbitrary activation function, such that $\bm{h}=f(Scale(\bm W \bm x_i))$ denotes the original output, and $\bm h'=f(Scale(\bm W' \bm x_i)) $ denotes the output after the weight matrix is scaled by $\delta$. Then we have $| \frac{\bm h'}{\bm h}-1| \le | \frac{\bm W'}{\bm W}-1|$.
\end{prop}
Since $| \frac{\bm h'}{\bm h}-1| \le | \frac{\bm W'}{\bm W}-1|$, the fluctuation of output caused by the weight scaling is alleviated.

\begin{figure}[t]
    \centering
    \includegraphics[width=0.5\columnwidth]{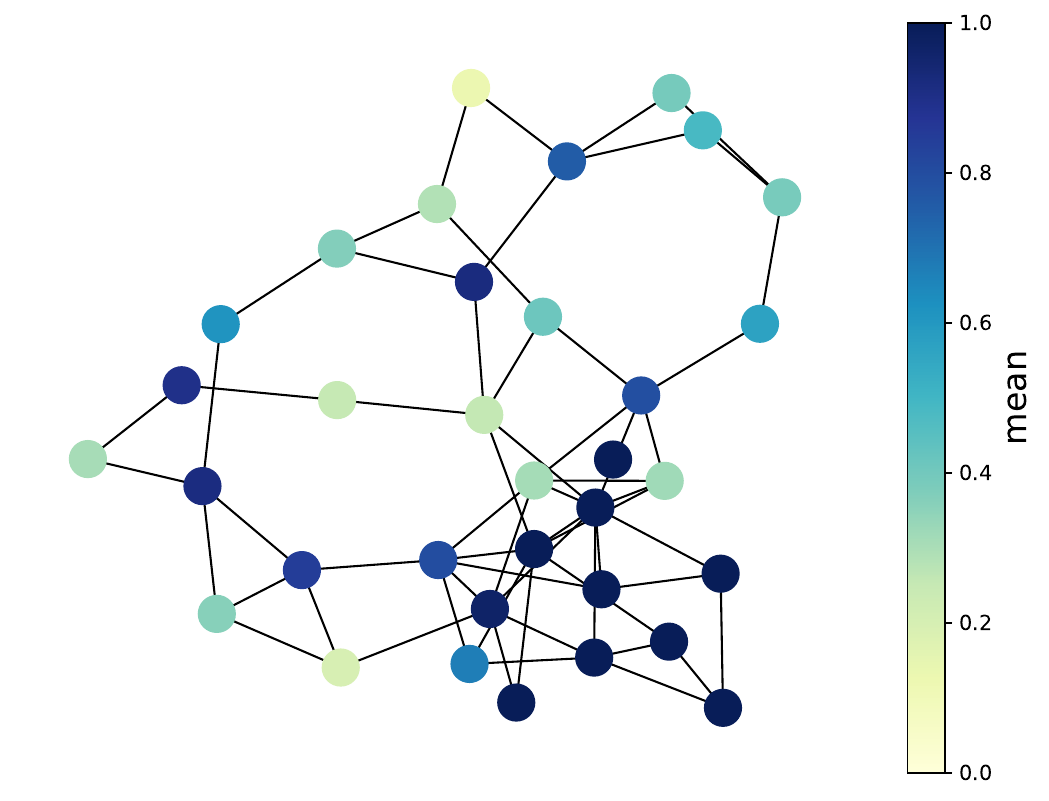}
    \caption{An illustration for showing that the mean statistics contains topology information using the KarateClub dataset. We can identify the nodes within the subgraph in the bottom right only using their means.}
    \label{fig:6}
\end{figure}

\subsection{Node-Wise Shifting for Alleviating Over-Smoothing}
Take GCNs as an example, $\mu_i = \sum_{j \in \mathcal{N}(i) \cup i}\frac{1}{\sqrt{d_i d_j}} \mu_j$, where $\mu_i = \frac{1}{d} \sum_{j=1}^d \bm x_{i,j}$. The mean statistics contain information about node degree and neighboring relationship. We visualize an example in Figure~\ref{fig:6}. The standard shift (shifting by subtracting the mean) may cause structural information loss. However, a recent work~\cite{DBLP:conf/icml/CaiLXHL021} suggests that subtracting the mean makes the optimization curvature of $\bm W$ smoother. They introduce a learnable parameter to achieve a trade-off between retaining structural information and boosting training. 
In this paper, we mainly focus on the relationship between over-smoothing and shift operation, which is less considered in previous work. To reveal the effects caused by shifting, we adopt the widely used metric of smoothness.
\begin{definition}
The smoothness score of the node embedding matrix $\bm X \in \mathbb{R}^{d \times n}$ is defined as 
\begin{equation}
\begin{split}
    Score(\bm X) &= tr(\bm X \bm{LX}^T) \\ &=\frac{1}{2}\sum_{i=1}^n \sum_{j\in \mathcal{N}(i) \cup i} \Vert \frac{\bm x_i}{\sqrt{1+d_i}} - \frac{\bm x_j}{\sqrt{1+d_j}} \Vert_2^2,\\
\end{split}
\end{equation}
where $\bm L = \bm I - \bm P = \bm I - \tilde{\bm D}^{-\frac{1}{2}}\tilde{\bm A}\tilde{\bm D}^{-\frac{1}{2}}$ is the graph laplacian.
\end{definition}
\newtheorem{Lemma}{Lemma}

A small value of the smoothness score implies that the node representations may suffer from over-smoothing. We remove the non-linear activation for simplifying the analysis.
\begin{Lemma}(Lemma 1 in~\cite{DBLP:conf/nips/ZhouHZCLCH21})
    The Smoothness score of the output of a GCN layer is bounded as $S_\text{min} C \leq Score(\bm {WXP})\leq S_\text{max} C$, where $C=tr(\bm {XPL}\bm{P}^T\bm X^T)$, $S_\text{min}$ and $S_\text{max}$ denote the minimum and maximum singular values of $\bm W$, respectively.
\end{Lemma}

If we formulate the shift in a matrix form, then applying shift operation is identical to multiply $\bm S$ in a form of $\bm{S (W x_i)}$, where $\bm S=\bm I_d - \frac{1}{d} \bm 1 \bm {1^T}$ and $\bm 1$ is a d-dimensional vector full of 1. We further check how this operation affects the spectrum of $\bm W$ and establish the following result. 

\newtheorem{theorem}{Theorem}
\begin{theorem}
Let $\eta_1$ and $\eta_d$ respectively be the minimum and maximum singular values of matrix $\bm{W}$, $\kappa_1$ and $\kappa_d$ respectively be the minimum and maximum singular values of matrix $\bm{SW}$. Then we have $\eta_1 \ge \kappa_1$ and $\eta_d \ge \kappa_d$.
\label{thm:1}
\end{theorem}
The above result demonstrates that the standard shift serves as a preconditioner of $\bm W$ that shrinks the maximum and minimum singular values, leading to smaller lower and upper bounds of the smoothness score. Hence, the standard shift potentially increases the risk of over-smoothing.

Based on the analysis, the standard shift induces structural information loss and accelerates over-smoothing, and therefore is not suitable for GRL. The importance of the mean statistics is different for each node, thus it is reasonable to retain a specific fraction of the mean for each node representation. On the other hand, using specific parameters for nodes with different degrees is considered a feasible way in prior works~\cite{tang2020investigating,DBLP:conf/kdd/WuHX19}. However, the degree-specific parameter strategy is too expensive due to the long-tailed distribution of node degrees. Although we can utilize the hashing-based approaches to map some degrees to the same entry of a hash table so as to reduce the parameters, it is still affordable for a normalization layer which is considered a light module. Here we utilize the shift operation to simulate the degree-specific parameter strategy. Concretely, we assign a weight matrix $\bm W_i = \bm S_i \bm W= (\bm I_d - \frac{k_i}{d}\bm 1 \bm {1^T})\bm W$ to node $v_i$, where $0<k_i<1$ and $1-k_i$ is the preservation ratio of mean for $v_i$. 
\newtheorem{remark}{Remark}
\begin{remark}
    the minimum and maximum singular values of matrix $\bm S_i \bm W$ are larger than those of matrix $\bm{SW}$.
\end{remark}

\begin{table*}[t]
\centering
\caption{Summary of classification accuracy results averaged over 10 independent trials. We highlight the best results in bold and the runner-up scores with underline.}
\label{tab:2}
\begin{tabular}{c|c|cccccccccc}
\midrule
Model &Normalization &Cora &Citeseer &Pubmed &Co-CS &Co-Physics &A-Comp. &A-Photo &Texas &ogbn-arXiv &Avg. rank\\
\hline
\multirow{5}{*}{GCN} 
&None &81.6 &69.3 &\uline{79.4} &90.8 &93.5 &82.7 &91.7 &58.2 &67.3 &3.44\\
&NodeNorm &\uline{81.8} &\uline{70.1} &78.1 &\uline{91.7} &93.8 &\uline{82.8} &\textbf{92.5} &\uline{58.7} &67.1 &2.44\\ 
&GraphNorm &76.6 &63.1 &76.7 &90.8 &87.1 &79.1 &90.2 &44.8 &\textbf{69.1} &4.67\\
&PairNorm &75.1 &63.6 &76.0 &85.7 &83.4 &74.3 &88.8 &44.5 &64.9 &5.89\\
&DGN &81.5 &69.3 &78.0 &91.2 &\uline{94.0} &81.1 &91.8 &58.3 &67.6 &3.33\\
&ResNorm &\textbf{83.4} &\textbf{73.2} &\textbf{80.0} &\textbf{93.0} &\textbf{94.6} &\textbf{84.0} &\uline{92.4}  &\textbf{60.5} &\uline{68.6} &1.22\\

\hline
\multirow{5}{*}{GAT} 
&None &81.0 &69.0 &78.1 &89.3 &91.3 &\uline{84.2} &91.0 &57.5 &67.1 &3.44\\
&NodeNorm
&\uline{82.3} &\uline{70.0} &\uline{78.3} &90.3 &91.3 &82.5 &91.5 &59.1 &68.0 &2.89\\
&GraphNorm
&77.8 &63.7 &76.4 &83.1 &87.8 &77.6 &90.4 &47.3 &\textbf{69.6} &4.67\\
&PairNorm
&75.3 &64.5 &75.4 &83.9 &86.6 &74.3 &87.3 &40.5 &64.5 &5.78\\
&DGN
&80.9 &68.0 &76.4 &\uline{90.7} &\uline{91.4} &80.5 &\uline{92.1} &\uline{59.7} &68.3 &3.11\\
&ResNorm
&\textbf{83.6} &\textbf{72.4} &\textbf{78.9} &\textbf{91.6} &\textbf{92.3} &\textbf{84.5} &\textbf{92.5} &\textbf{60.2} &\uline{68.4} &1.11\\
\hline
\multirow{5}{*}{SAGE} 
&None &80.0 &69.1 &77.2 &88.4 &92.4 &\uline{81.4} &90.0 &\uline{68.8} &66.9 &3.33\\
&NodeNorm
&\uline{81.9} &\uline{70.8} &\uline{77.9} &\uline{91.9} &93.2 &81.1 &\uline{92.0} &67.5 &66.7 &2.89\\
&GraphNorm
&76.1 &62.4 &75.6 &85.7 &84.7 &77.5 &88.9 &67.8 &\textbf{69.0} &4.44\\
&PairNorm
&74.0 &61.9 &74.4 &80.1 &82.0 &70.2 &86.5 &47.8 &64.4 &6.00\\
&DGN
&80.6 &68.2 &76.9 &89.9 &\textbf{93.4} &80.2 &90.8 &68.1 &\uline{67.0} &3.00\\
&ResNorm 
&\textbf{82.6} &\textbf{72.9} &\textbf{79.3} &\textbf{93.2} &\uline{93.3} &\textbf{83.4} &\textbf{92.5} &\textbf{69.4} &\uline{67.0} &1.33\\
\hline
\midrule

\end{tabular}
\end{table*}

Motivated by the fact that larger degrees are generally associated with larger NStds, we set $k_i = 1-\frac{\sigma_i}{\sigma_{max}}$, where $\sigma_{max}$ is the maximum NStd of all nodes. The preservation ratio of the mean for $v_i$ is $\frac{\sigma_i}{\sigma_{max}}$, which is proportional to the NStd. Mathematically, the shift component of ResNorm is defined as $\text{Shift}(\bm x_i) = \bm x_i -  (1-\frac{\sigma_i}{\sigma_{max}}) \mu_i$.
Combining the two components, ResNorm can be formulated as
\begin{equation}
    \text{ResNorm}(\bm{x_i}) = \frac{(\bm x_i - (1-\frac{\sigma_i}{\sigma_{max}})\mu_i) \cdot \overline{\sigma^{\gamma}}}{\sigma_i ^ {e+\gamma}}.
\end{equation}
Unlike the common setting, we do not introduce the learnable affine parameters that may lead to over-fitting since the shift and scale of ResNorm are applied node-wisely rather than channel-wisely.

\section{experiments}\label{sec:6}
This section evaluates the effectiveness of ResNorm. 
\subsection{Experimental Settings}
Following the literature, we adopt the following widely-used datasets: three citation networks including Cora, Citeseer, Pubmed~\cite{sen2008collective}, two co-author networks including Coauthor Physics and Coauthor CS,~\cite{shchur2018pitfalls}, two Amazon product networks including Amazon-computers and Amazon-Photo~\cite{shchur2018pitfalls}, a heterophilious dataset Texas~\cite{pei2020geom}, a large-scale dataset ogbn-arXiv~\cite{DBLP:journals/qss/WangSHWDK20}, and three social networks including Squirrel~\cite{DBLP:conf/iclr/PeiWCLY20}, Actor~\cite{DBLP:conf/iclr/PeiWCLY20}, and Email~\cite{DBLP:conf/kdd/YinBLG17}. The statistical information of the datasets is presented in Table\ref{tab:dataset_stat}.
\begin{table}
    \centering
        \caption{Dataset statistics.
}
    \begin{tabular}{c c c cc}
    \hline
    &\textbf{\#Nodes} &\textbf{\#Edges} &\textbf{\#Features} &\textbf{\#Classes} \\
       \hline
    \textbf{Cora} &2,708 &5,278 &1,433 &6 \\
    \textbf{Citeseer}  &3,327 &4,676 &3,703 &7 \\
    \textbf{Pubmed} &19,717 &44,327 &500 &3 \\

    \textbf{Coauthor CS} &18,333
&163,788
&6,805
&15 \\
    \textbf{Coauthor Physics}&34,493
&495,924
&8,415
&5  \\
    \textbf{Amazon Computers} &13,752
&491,722
&767
&10 \\
    \textbf{Amazon Photo} &7,650
&238,162
&745
&8 \\
    \textbf{Texas}  &183 &295 &1,703 &5 \\
    \textbf{ogbn-arxiv} &169,343	&1,166,243&128	&40 \\
     \textbf{Actor}  &7,600 &26,752 &931 &5\\
    \textbf{Squirrel}  &5,201 &198,493 &2,089 &5\\
       \textbf{Email} &1,005 &25,571 &128 &42 \\
    \hline

    \end{tabular}
    \label{tab:dataset_stat}
\end{table}
We implement all the methods and GNN backbones with Pytorch Geometric Library~\cite{DBLP:journals/corr/abs-1903-02428}. Experiments are conducted using a single 24G NVIDIA GEFORCE RTX 3090. We perform a grid search to tune the hyper-parameters. The search space for each hyper-parameter is as follows: dropout ratio $\in \left[0,0.9 \right]$, $\beta \in \left[0,0.7\right]$, $\alpha \in \left[0,0.7\right]$, hidden units $\in \left\{ 32,64,128\right\}$, weight decay $\in \left\{ 5e-4,1e-4,1e-5\right\}$, learning rate $\in \left\{ 0.01,0.02,0.005\right\}$. For training, we select the Adam optimizer and the best model checkpoint with the minimum validation loss. 

\subsection{Baselines}
The baselines can be divided into two categories.

\textbf{Graph-based Normalization} Since no previous work was dedicated to addressing the long-tailed issues via normalization. We compare ResNorm (RN) with 4 graph-based normalization methods, including NodeNorm (NN)~\cite{DBLP:conf/cikm/ZhouDWLHXF21}, GraphNorm (GN)~\cite{DBLP:conf/icml/CaiLXHL021}, PairNorm (PN) ~\cite{DBLP:conf/iclr/ZhaoA20}, and Differentiable Group Normalization (DGN)~\cite{DBLP:conf/nips/Zhou0LZCH20}. We exclude the results of the most famous normalization methods BatchNorm and LayerNorm as they have been shown to be significantly outperformed by GraphNorm in ~\cite{DBLP:conf/icml/CaiLXHL021}. 
For a comprehensive evaluation, all the methods are evaluated over the three most representative GNN frameworks, \textit{i}.\textit{e}., GCN~\cite{DBLP:conf/iclr/KipfW17}, the upgraded version of graph attention networks (GATv2)~\cite{DBLP:journals/corr/abs-2105-14491}, and a sampling-based approach GraphSage~\cite{DBLP:conf/nips/HamiltonYL17}. 

\textbf{Tail Refinement GNNs} Additionally, we incorporate four methods specifically aimed at enhancing the performance of tail nodes: Tail-GNN\footnote{\url{https://github.com/shuaiOKshuai/Tail-GNN}}~\cite{DBLP:conf/kdd/ZhengMGZ21}, meta-tail2vec~\cite{DBLP:conf/cikm/LiuZFZH20}, DEMO-Net~\cite{DBLP:conf/kdd/WuHX19}, and role2vec~\cite{DBLP:journals/tkde/AhmedRLWZKE22}. Tail-GNN explicitly transfers the information from head to tail. Meta-tail2vec
employs a two-stage framework for robust tail node embedding. DEMO-Net differentiates nodes of varying degrees by employing degree-specific transformations, while role2vec distinguishes node roles based on structural features, including degree. They are all different from the ResNorm which reshapes the NStd distribution.

Specifically, we apply the normalization as follows,
\begin{equation}
   \mathbf{H}^{(l+1)} = \text{Relu}(\text{NormLayer}(\text{ConvLayer}(\mathbf{H}^{(l)}, \mathbf{A}))),
\end{equation} where ConvLayer denotes a type of graph convolution layer (GCN, GAT, or GraphSAGE), NormLayer denotes a type of normalization layers, $\mathbf{H}^{(l)}$ is the node representation matrix at the $l$-th layer, and $\mathbf{A}$ is the adjacency matrix.

\begin{table}[t]
\centering
\caption{Average test accuracy of tail nodes over 10 trails on the Squirrel, Actor, Email, and Coauthor-CS datasets. Note that the results of baselines are taken from~\cite{DBLP:conf/kdd/ZhengMGZ21}}.
\label{tab:tail-gnn}
\begin{tabular}{c|c|c|c|c}
\midrule
 Method &Squirrel &Actor &Email &Coauthor-CS\\
 \midrule
 
 GCN  &24.8   &29.7 &57.9  & 88.4  \\
 \textbf{ResNorm-GCN} &\textbf{37.1} &29.3 &59.8 &93.0\\
DEMO-Net &28.3  &28.4  &56.9 &90.8 \\
role2vec &26.3  &23.1 &44.9 &62.7 \\
 Meta-tail2vec &25.1 &29.7 &57.1 &89.3 \\
Tail-GCN &29.3 &34.8 &59.2 &\textbf{93.7}\\
\textbf{Tail-GNN+ResNorm} &34.2 &\textbf{35.5} &\textbf{60.1} &93.6\\
\midrule    
\end{tabular}
\end{table}

\subsection{Performance Results}
The node classification results regarding normalization methods are reported in Table \ref{tab:2}, from which we can make two observations. (1) The results show that ResNorm outperforms other normalization methods in most settings, which demonstrates the effectiveness of ResNorm on node classification tasks. Although the preliminary analysis (e.g., Sec.\ref{sec:5}) is mostly based on GCN, the consistent performance improvement validates the effectiveness of ResNorm in working with prevalent GNN frameworks. It is worth noting that although GraphSage utilizes a sampling-based strategy to perform graph convolution, large-degree nodes still can receive information from more distinct neighboring nodes in expectation. We argue that the higher test accuracy of ResNorm comes from the accuracy improvement of the major tail nodes. 
(2) ResNorm is the only normalization method that can consistently achieve better performance compared to GNN backbones (none normalization). GraphNorm accelerates the training of GNNs with the cost of accuracy. Although GraphNorm performs well on ogbn-arXiv, the average rank of GraphNorm is around 4.5 on all the datasets. PairNorm and DGN are tailor-made for addressing over-smoothing. However, they cannot enable deep GNNs to outperform their shallow counterparts. Hence, they achieve their best performance with 2 layers and the performance is worse than the GNN backbones.

The results for tail nodes are presented in Table \ref{tab:tail-gnn}, following the same definition of tail node (degree $< 5$) and experimental settings as \cite{DBLP:conf/kdd/ZhengMGZ21}. It is important to note that these models can often be complex, incorporating multiple components and auxiliary losses. As a result, they encounter challenges when scaling to large graphs. For instance, Tail-GCN requires over 20 minutes to converge on Coauthor-CS when utilizing an Nvidia P100 GPU. In contrast, ResNorm-GCN achieves convergence in less than 2 minutes. The results in Table \ref{tab:tail-gnn} demonstrate the effectiveness of ResNorm, as ResNorm-GCN outperforms DEMO-Net, role2vec, and Meta-tail2vec across the employed datasets in terms of node classification accuracy for tail nodes. Although GCN+ResNorm generally performs worse than Tail-GNN, incorporating ResNorm into Tail-GNN leads to further improvements. Specifically, Tail-GNN+ResNorm achieves approximately a 16\% improvement over Tail-GNN on the Squirrel dataset. This provides further evidence that ResNorm is effective and can be integrated to further enhance existing models.

\begin{figure*}[t]
    \centering
          \subfigure[GCN (Cora)]
    {\includegraphics[width=0.45\columnwidth]{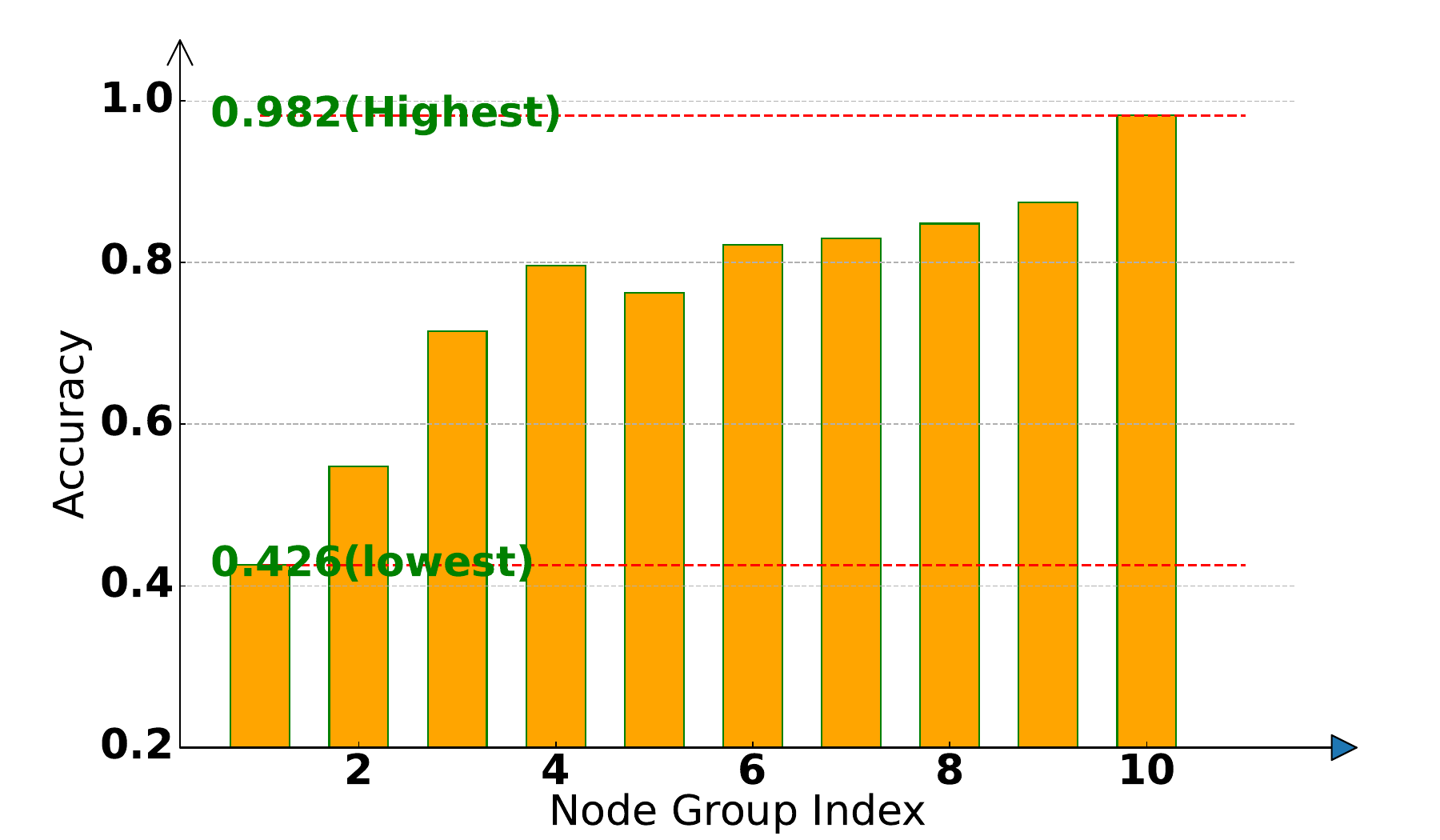}}
    \subfigure[GCN+ResNorm (Cora)]
    {\includegraphics[width=0.45\columnwidth]{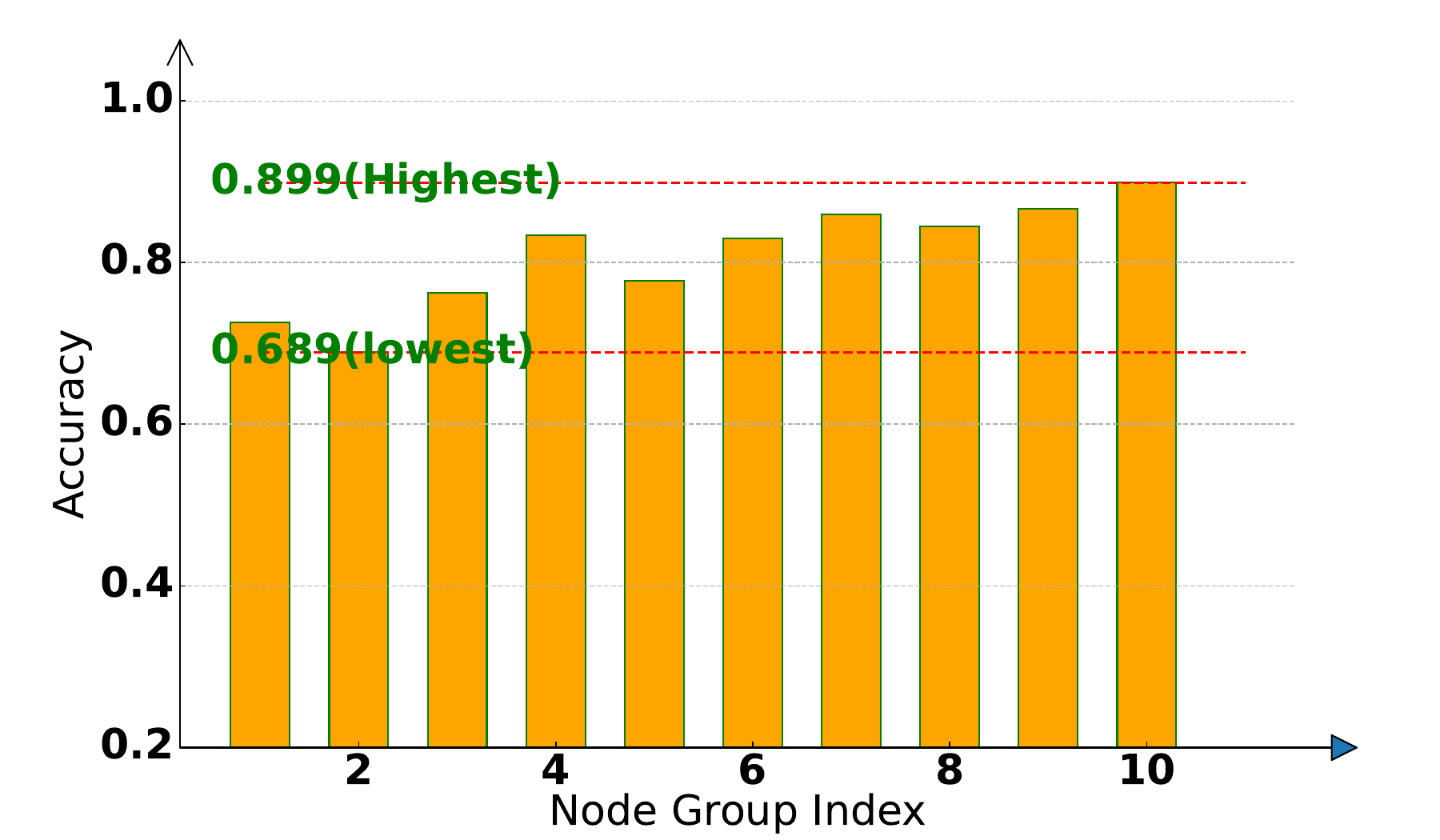}}
        \subfigure[GCN (Pubmed)]
    {\includegraphics[width=0.45\columnwidth]{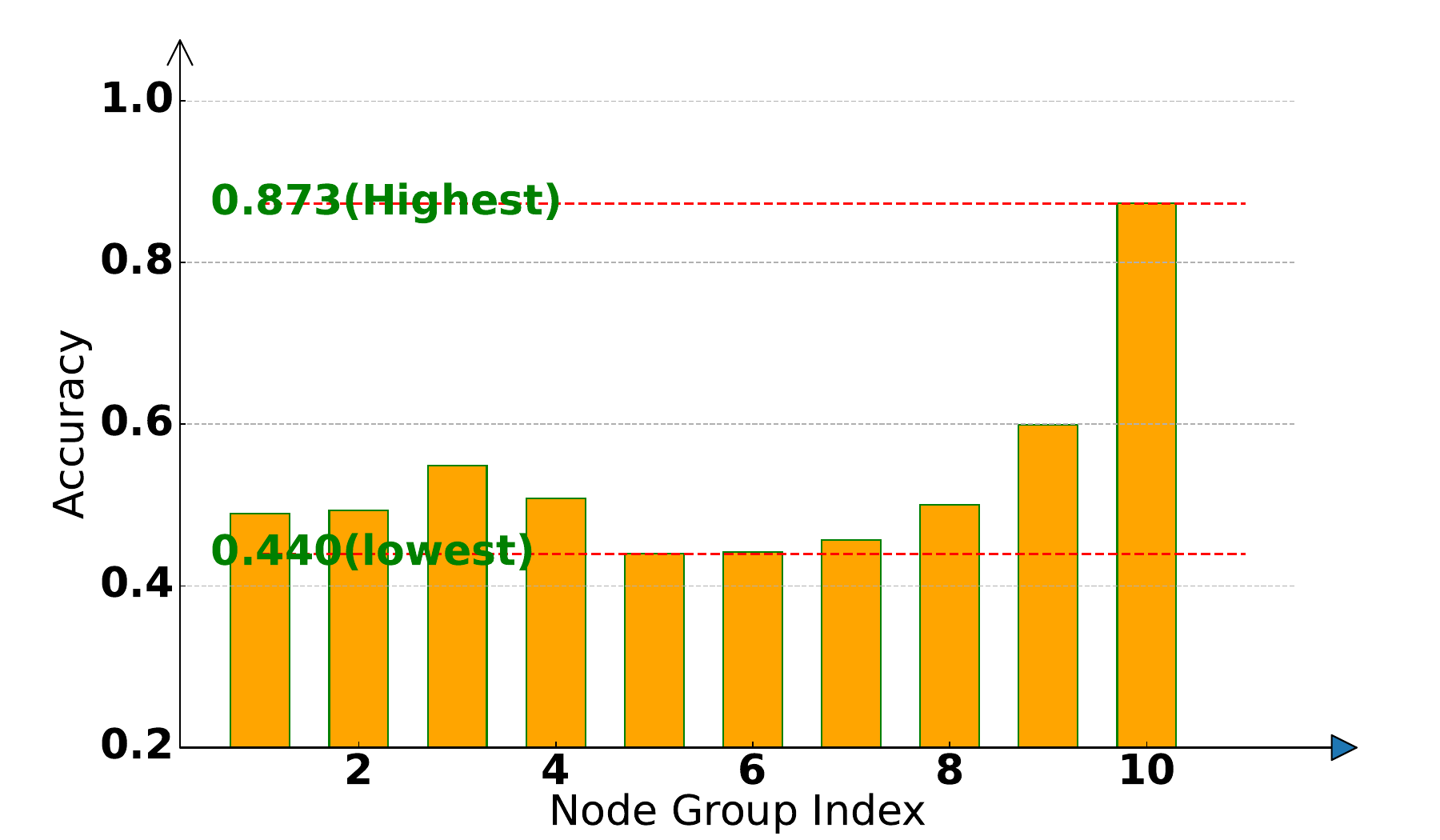}
    }
    \subfigure[GCN+ResNorm (Pubmed)]
    {\includegraphics[width=0.45\columnwidth]{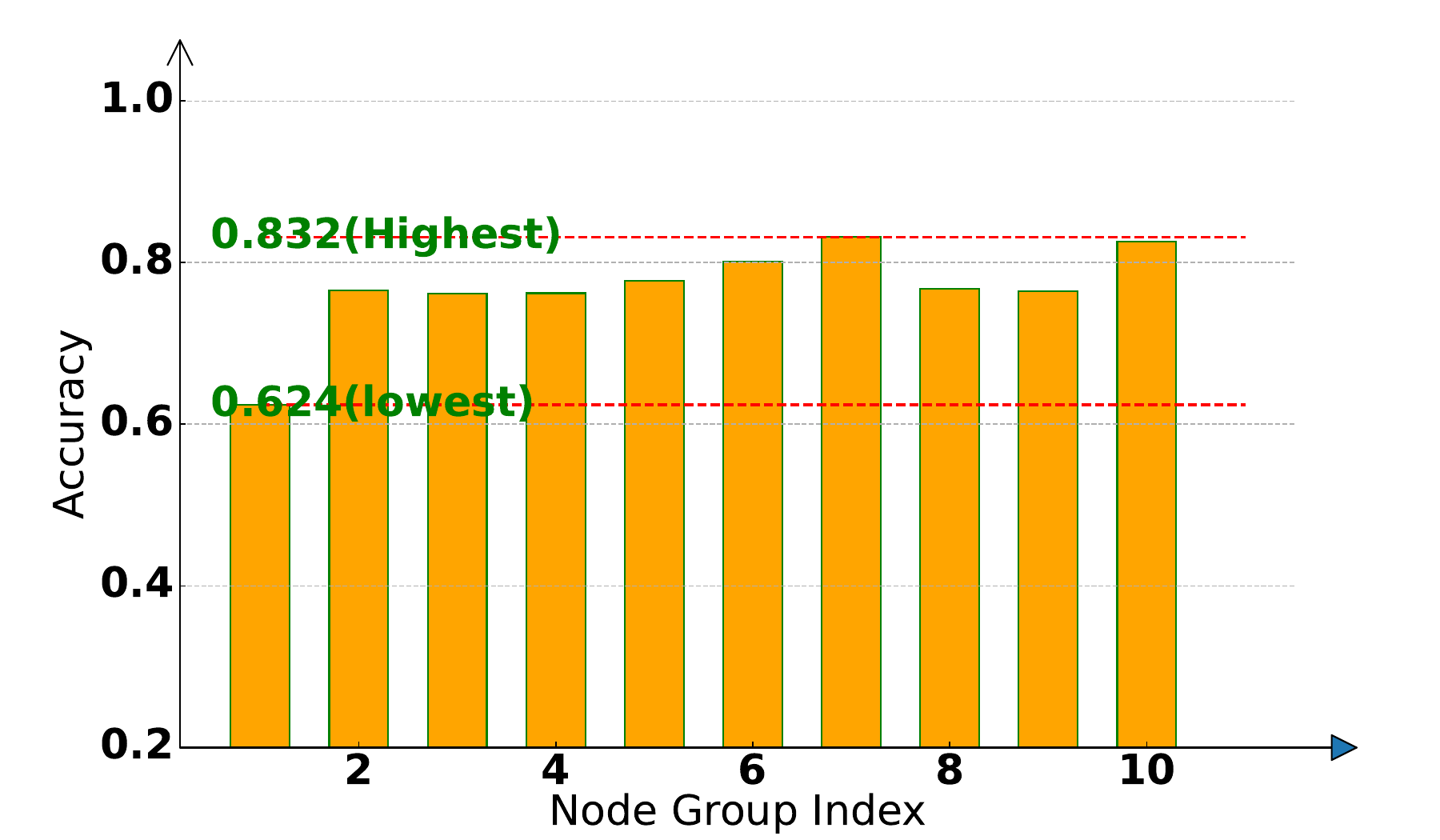}
    }
    \caption{Accuracy of different node groups on different datasets, where indices of node groups are ranked from smaller NStd to larger NStd.}
    \label{fig:8}
\end{figure*}

\subsection{Accuracy of Tail Nodes}
Here, we study the average test accuracy of different node groups split by their NStds. For a better illustration, we use 8-layer GCNs for the experiment. The node groups are ranked from smaller to larger NStd (the nodes in the group of index 0 have the smallest NStds). 
Figure~\ref{fig:8} depicts the average test accuracy of these 10 groups. The gap of accuracy among different node groups is significantly smaller when GCN is equipped with ResNorm. Particularly, the lowest accuracy bars of GCN+ResNorm are significantly higher than that of vanilla GCN. For example, the lowest bar of GCN+ResNorm on Cora is 68.9\% while that of GCN is only 42.6\%. The results demonstrate that ResNorm can significantly improve the performance of the tail nodes, which compensates for the declines of the highest accuracy bars. Consequently, the overall performance is improved. Note that for each node, a false prediction implies that GNN cannot learn a ``correct'' representation for it. Hence, ResNorm is essentially balancing the quality of learned representations across those representative and unrepresentative nodes.

\subsection{Hyperparameter Analysis}
\begin{figure}
    \centering
    \subfigure[]{\includegraphics[width=0.48\columnwidth]{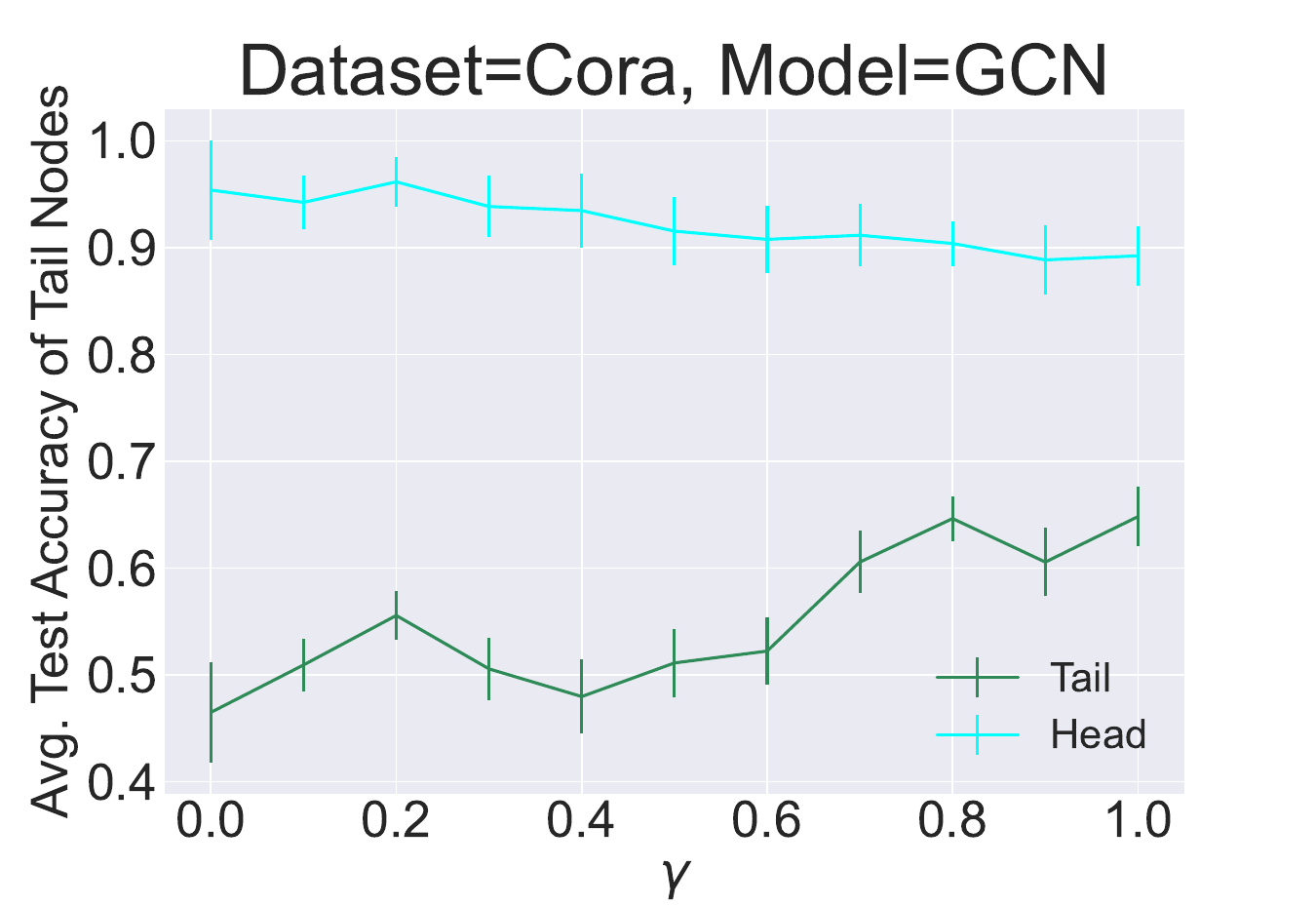}}
    \subfigure[]{\includegraphics[width=0.48\columnwidth]{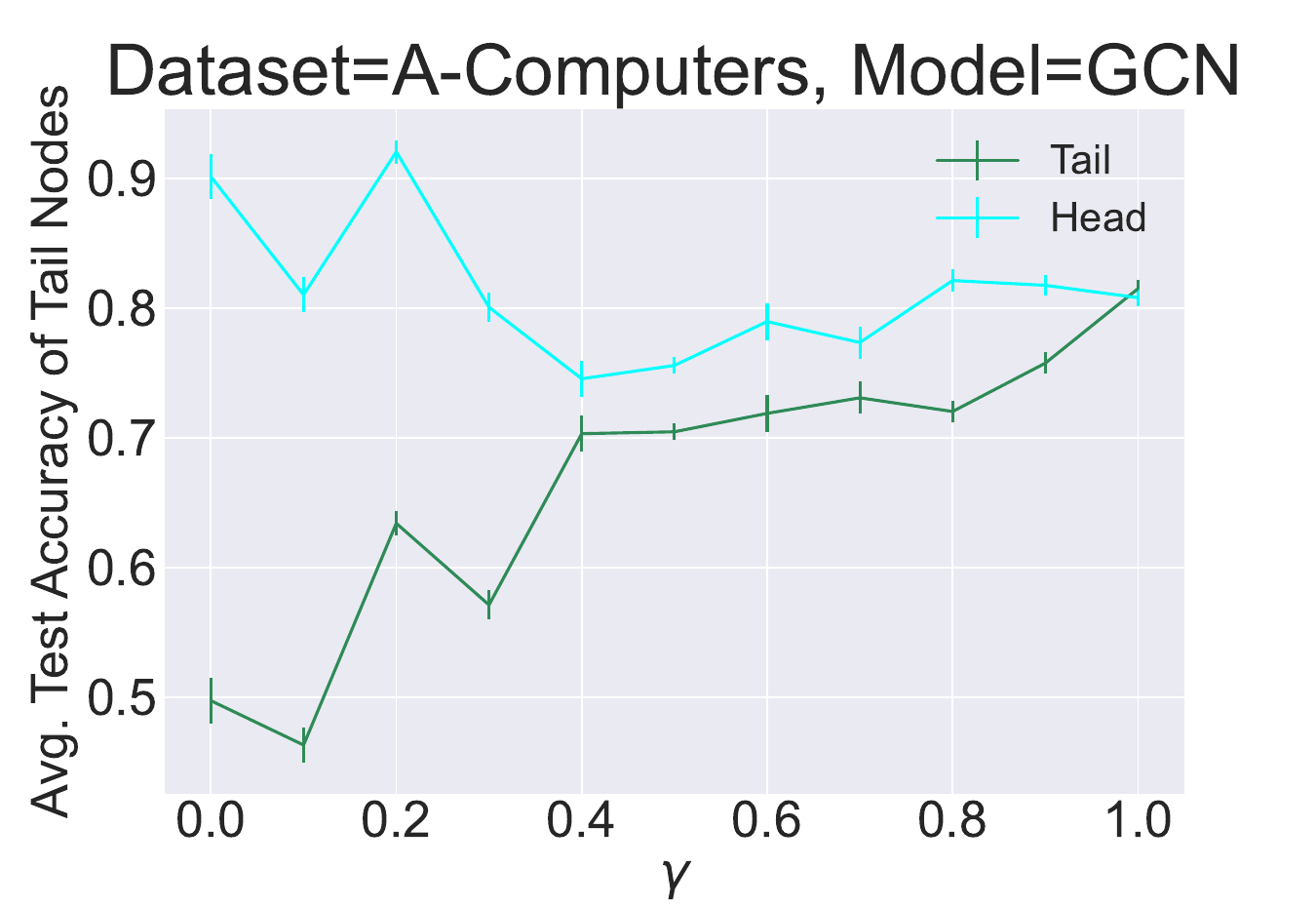}}
    \caption{Test accuracy of tail nodes w.r.t. $\gamma$. The error bar indicates a 95\% confidence interval.}
    \label{fig:lambda}
\end{figure}
To shed light on the working mechanism of ResNorm. we provide an analysis w.r.t. the parameter $\gamma$ based on the test accuracy of tail nodes (nodes with the smallest 10\% NStds) and head nodes (nodes with the largest 10\% NStds). Note that the larger the $\gamma$, the smaller the difference of NStds between the tail and head nodes. Intuitively, the difference in accuracy between the tail and head nodes also reduces. As shown in Figure \ref{fig:lambda}, as $\gamma$ increases, the accuracy of tail nodes gradually increases while that of head nodes gradually decreases, which conforms to our previous analysis that $\gamma$ can balance the performance gains of tail nodes and the performance drops of head nodes. We empirically found that $\gamma$ is related to the number of graph convolution layers. Generally, a deeper model may need a larger $\gamma$ because as the model gets deeper, the difference of accuracy between tail and head nodes may also enlarge.

\begin{table}[t]
\centering
\caption{Test accuracy of deep GNN models on citation networks (RN denotes ResNorm here). We also calculate the average improvement achieved by ResNorm over each backbone. }
\label{tab:3}
\begin{tabular}{c|c|c|c|c}
\midrule
Dataset &Model &Layer 16 &Layer 32 &Impro.(\%)\\

\midrule
\hline
\multirow{8}{*}{Cora} &GCN+PN
&62.5   &42.1 &\multirow{2}{*}{32.5}\\
&GCN+PN+RN
&\textbf{77.5} &\textbf{59.3} \\
\cline{2-5}
&APPNP &82.4 &82.4 &\multirow{2}{*}{3.1}\\
&APPNP+RN &\textbf{84.9} &\textbf{85.1}\\
\cline{2-5}
&DeepGCN &69.5   &57.4  &\multirow{2}{*}{21.5}\\
&DeepGCN+RN &\textbf{78.6}   &\textbf{74.7}\\
\cline{2-5}
&GCNII &84.6   &85.1  &\multirow{2}{*}{0.8}\\
&GCNII+RN &\textbf{85.4}   &\textbf{85.6}\\
\hline
\multirow{6}{*}{Citeseer} &GCN+PN
&39.0   &18.1 &\multirow{2}{*}{90.0}\\
&GCN+PN+RN
&\textbf{63.5} &\textbf{39.3} \\
\cline{2-5}
&APPNP &69.6 &69.7 &\multirow{2}{*}{7.0}\\
&APPNP+RN &\textbf{74.3} &\textbf{74.8}\\
\cline{2-5}
&DeepGCN &50.7	&41.9 &\multirow{2}{*}{36.9}\\
&DeepGCN+RN &\textbf{63.3}   &\textbf{62.4}\\
\cline{2-5}
&GCNII &72.2   &73.4  &\multirow{2}{*}{0.8}\\
&GCNII+RN &\textbf{72.8}   &\textbf{74.0}\\
\midrule    

\end{tabular}
\end{table}

\subsection{Evaluation on Deep GNNs}
To further investigate the efficacy of ResNorm, we evaluate its performance on state-of-the-art deep models. We employ four deep models as backbones, including APPNP~\cite{DBLP:conf/iclr/KlicperaBG19}, DeepGCN~\cite{DBLP:conf/iccv/Li0TG19}, GCN+PairNorm, and GCNII~\cite{DBLP:conf/icml/ChenWHDL20}. For APPNP, ResNorm is only applied at the final output since the propagation process has no parameter. We tune the strength of initial residual $\alpha \in \left\{0.15,0.2\right\}$. For GCN+PairNorm, ResNorm is applied right after PairNorm at each layer. We select the implementation of residual connection block for DeepGCN in Pytorch Geometric. We report the average test accuracy of 10 trials in Table \ref{tab:3}.
The results demonstrate that ResNorm, serving as a plug-and-play and light module, can significantly improve deep GNN models, even though the models already have been improved by other normalization technique that addresses the over-smoothing. For example, ResNorm+PairNorm achieves 32.5\% improvement over PairNorm on Cora. In particular, ResNorm+APPNP outperforms APPNP by a large margin, which validates that ResNorm is a practical normalization layer that can be used to enhance deep models. 
Most existing deep models only focus on the over-smoothing issue. However, the long-tailed degree distribution also degrades the performance of GNNs. We note that applying ResNorm to existing models is a practical way to address this issue.
\subsection{Ablation Study on the Shift Operation}
\begin{figure}
    \centering
    \subfigure[Cora]{\includegraphics[width=0.48\columnwidth]{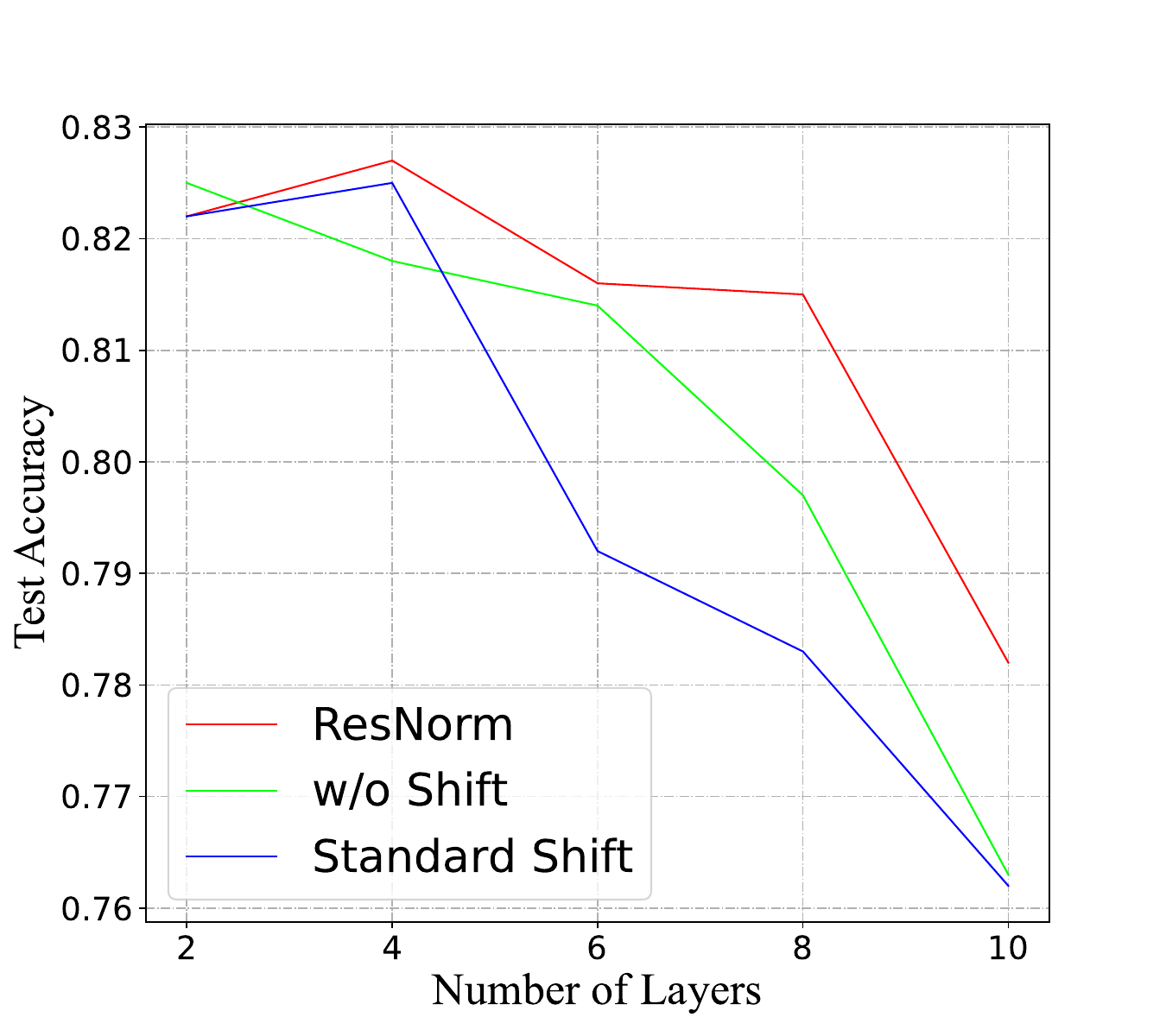}}
    \subfigure[Coauthor-CS]{\includegraphics[width=0.48\columnwidth]{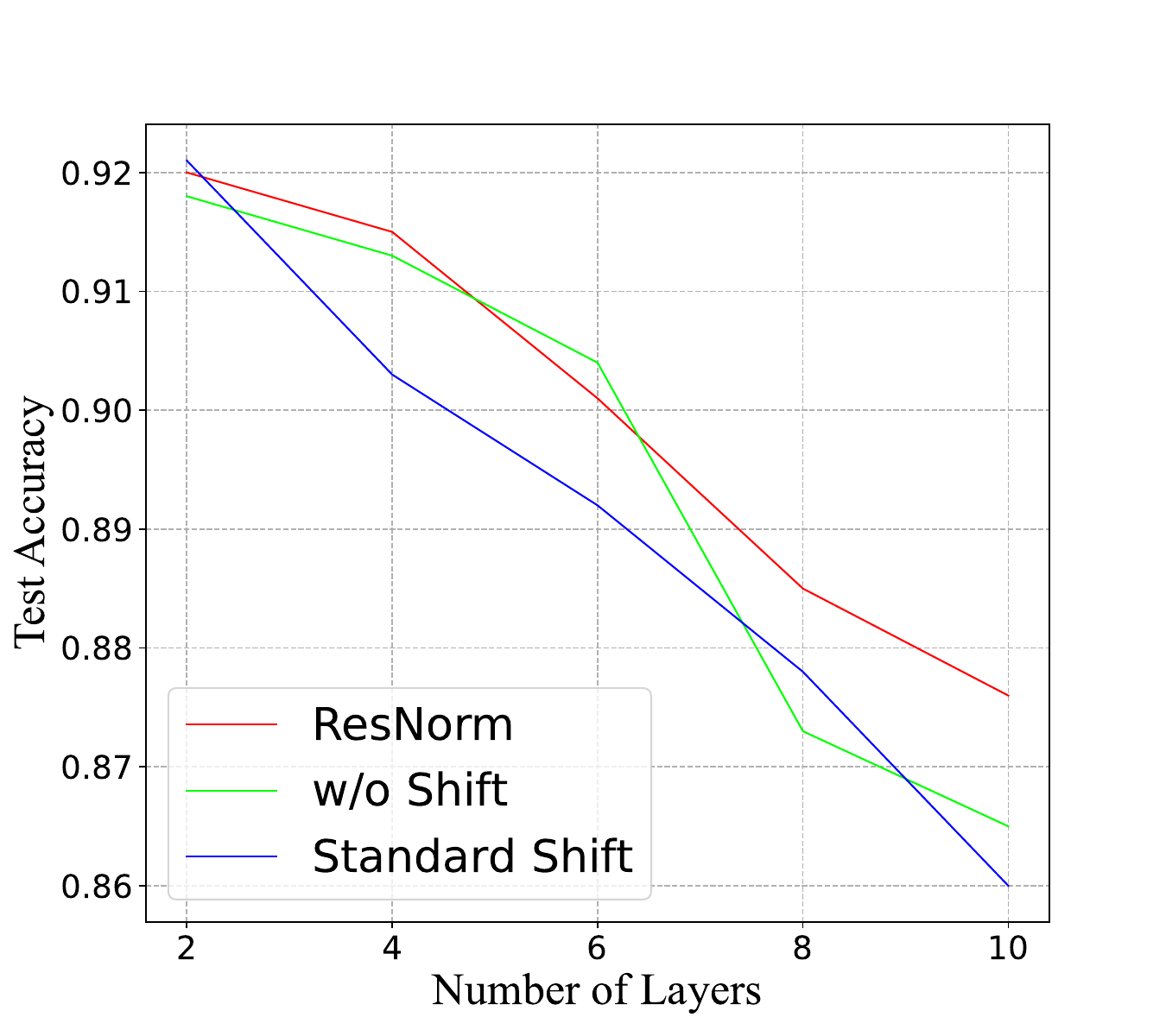}}
    \caption{Test accuracy w.r.t different numbers of convolutional layers.}
    \label{fig:ablation}
\end{figure}
This section presents an ablation study focused on the shift operation. Since our primary concern is addressing the issue of over-smoothing in the design of the shift operation, we examine the performance of ResNorm and its variants with varying numbers of layers. We explore two variants: the first one involves eliminating the shift operation from ResNorm, while the second variant replaces the shift operation with a standard shift that subtracts the mean from node representations. The empirical results presented in Figure \ref{fig:ablation} support our previous theoretical analysis, confirming that the use of standard shift can increase the risk of over-smoothing. It is observed that as the number of layers increases, the blue lines exhibit a significant drop compared to the red and green lines. Additionally, the red lines generally exhibit higher values compared to the green lines, further validating the efficacy of the node-wise shifting."

\section{Conclusion}\label{sec:7}
This paper studies tackling the long-tailed degree distribution issue in GNNs via normalization. We observe the variance expressing phenomenon and give explanations for why the head nodes commonly have larger NStd. We propose to leverage the distribution reshaping technique to improve the accuracy of tail nodes. The distribution reshaping operation constitutes the scale of ResNorm, aiming to implicitly reshape the distribution of node degree. We give an explanation of its mechanism using the influence distribution. Regarding the shift operation, we theoretically prove that the standard shift serves as a preconditioner of the weight matrix that shrinks the singular values. Consequently, the standard shift increases the risk of over-smoothing. With the over-smoothing in mind, we utilize a node-wise shift operation that can simulate the degree-specific parameter strategy. Experimental results confirm the effectiveness of ResNorm on the node classification task. 

\section{Limitations and Future Work}
The theoretical and empirical findings presented in this paper are mainly based on the most prevalent model in the field, namely GCN. In recent years, many advanced models have been proposed, and considering their utilization for analysis can offer valuable insights. However, such analysis may entail complex computational procedures and experiments, which are left as future research directions. Additionally, while the node-wise shift aims at addressing the issue of over-smoothing, it does not possess the ability to accelerate convergence as the standard shift does. Therefore, it is beneficial to investigate the potential for simultaneously reducing over-smoothing and improving convergence speed.

\ifCLASSOPTIONcaptionsoff
  \newpage
\fi



\bibliographystyle{abbrv}
\bibliography{bibtex/bib/IEEEexample}

\clearpage
\appendix
\begin{figure}[h]
    \centering
    \subfigure[Coauther-CS]
    {\includegraphics[width=0.47\columnwidth]{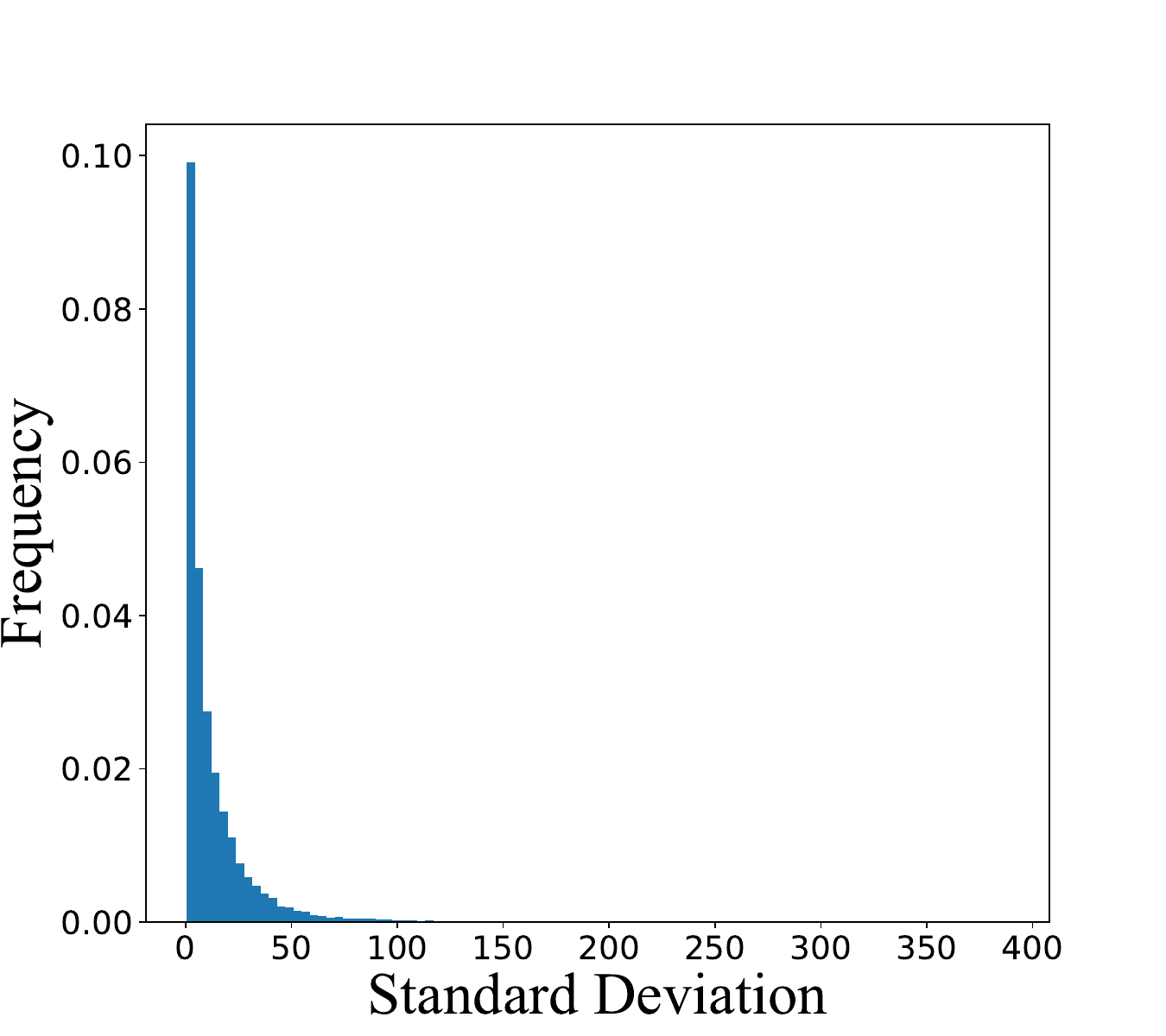}}
    \subfigure[Coauthor-Physics]
    {\includegraphics[width=0.47\columnwidth]{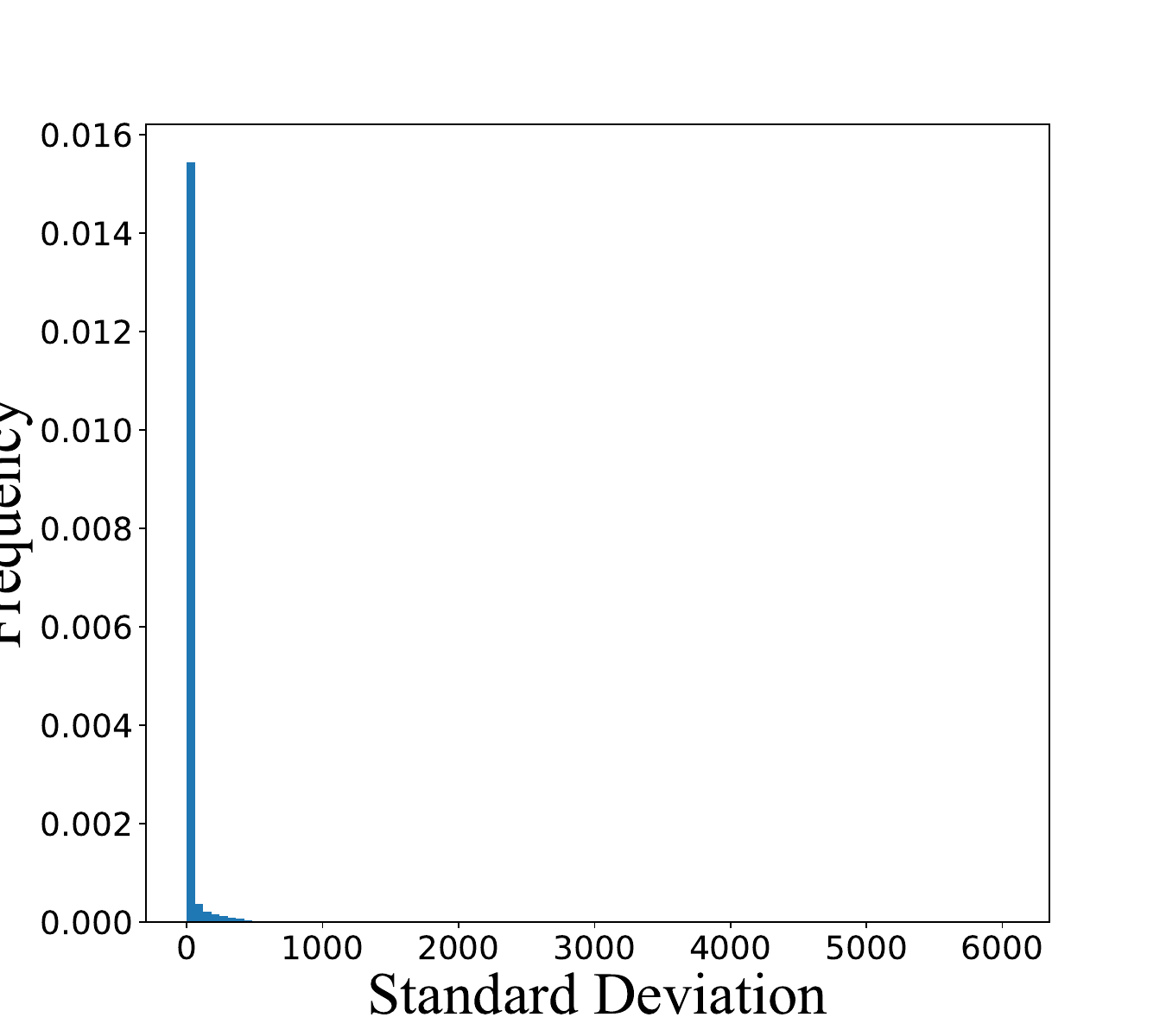}}
     \caption{NStd distributions.}
    \label{fig:nstd}
\end{figure}
\begin{figure}[h]
    \centering
    \subfigure[Coauthor-CS]
    {\includegraphics[width=0.47\columnwidth]{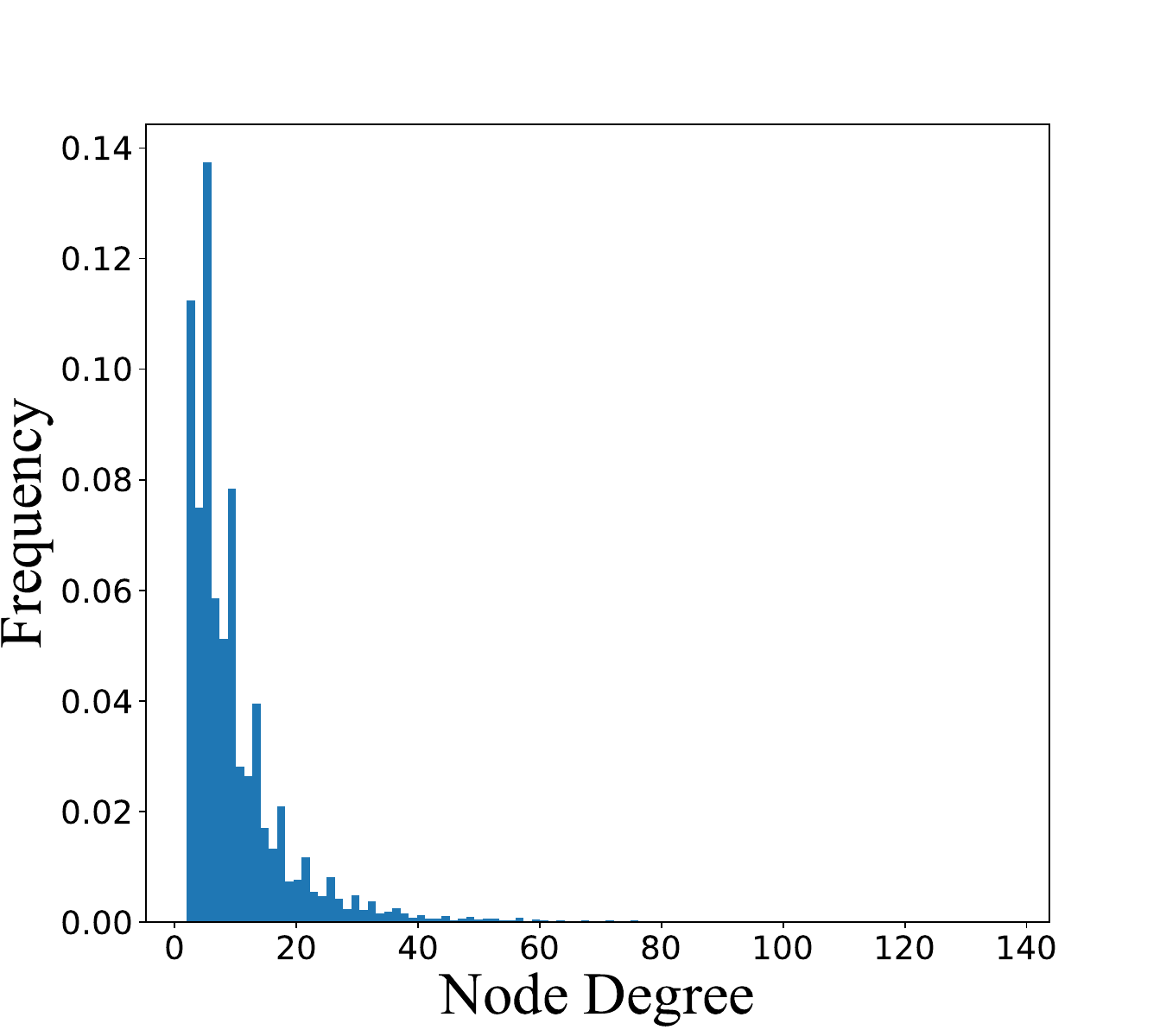}}
    \subfigure[Coauthor-Physics]
    {\includegraphics[width=0.47\columnwidth]{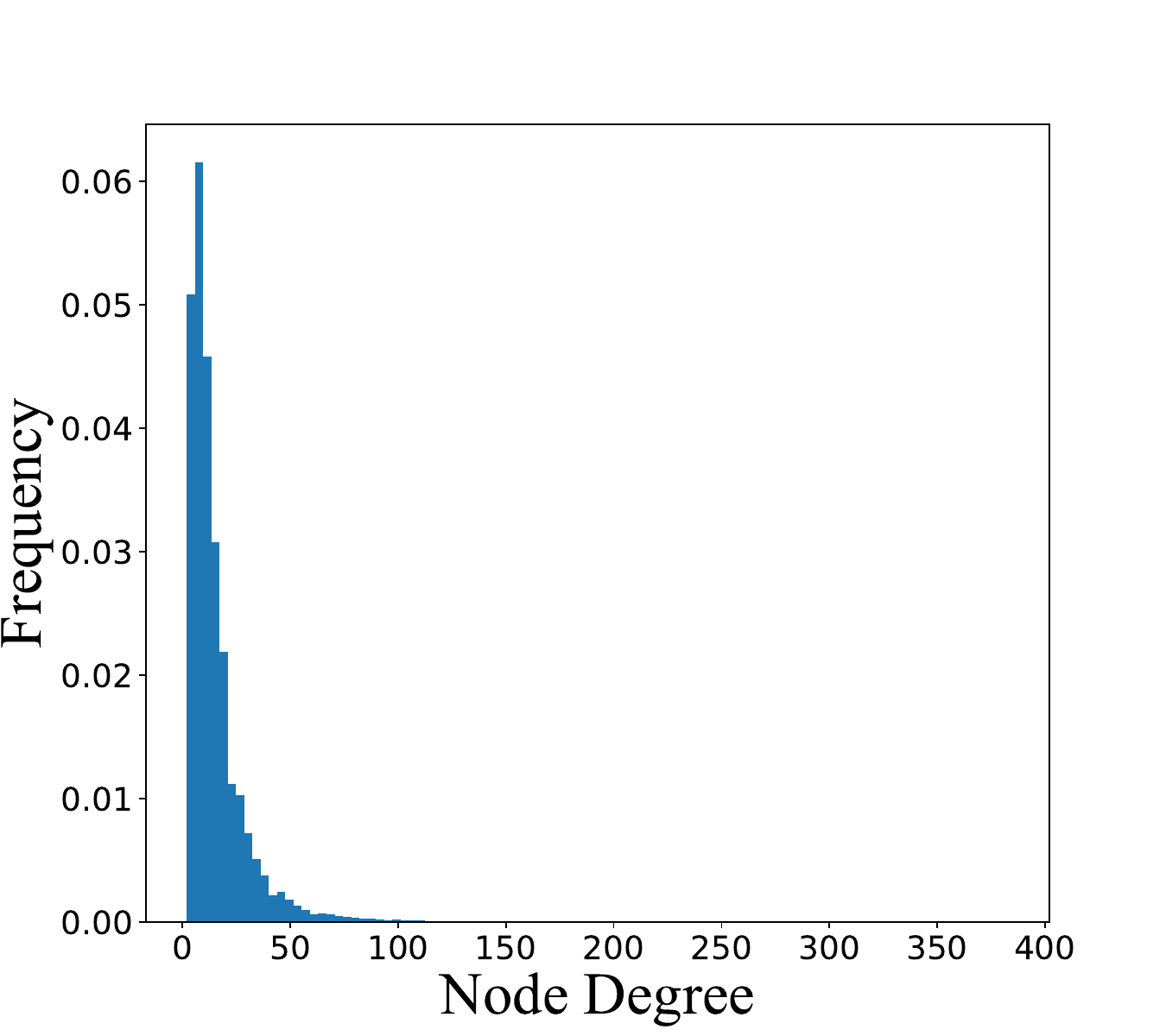}}
     \caption{Node degree distributions.}
    \label{fig:degree}
\end{figure}

\subsection{Proof of proposition \ref{theorem:1}}
\begin{proof}
We start by deriving the derivative between $\frac{ h_x}{\sigma_x^e}$ and $h_x$, where $\sigma_x$ denotes the std of d-dimensional vector $h_x$. Assuming $\sigma_x>0$, we have
\begin{equation} \label{Eq:jacobian_std}
    \begin{split}
        \frac{\mathrm{d}  (h_x\cdot \frac{1}{\sigma_x^e})}{\mathrm{d} h_x} &= \frac{\mathrm{d} (h_x \cdot (\frac{1}{d} h_x^T h_x - \frac{1}{d^2}\bm 1^T h_x \bm 1^T h_x)^{0.5e})}{\mathrm{d} h_x}\\
        &=\frac{1}{\sigma_x^e}\bm I_d - \frac{e}{d \sigma_x^{1+0.5e}} h_x(h_x^T - \mu \bm1^T),
    \end{split}
\end{equation}
where $\sigma_x$ is the standard deviation of $h_x$, $0<e<1$ a hyper-parameter, $\bm 1$ an all-ones vector and $\mu$ the mean of $h_x$.

Assume we apply the scale at each layer right after the output of convolution layer, \textit{i}.\textit{e}., $\bm H^{(l)} = Relu(Scale(GCN(\bm H^{(l-1)})))$ for any $1<l\leq k$. The output of the $l$-th convolution layer is denoted by $f_x^{(l)} = \frac{1}{\sqrt{deg_x}} \sum _{z\in \mathcal N(x)}\frac{1}{\sqrt{deg_z}}\bm W^{(l)} h_z^{(l-1)}$, and the output of scale is denoted by $B_x^{(l)} = f_x^{(l)}\cdot (\frac{1}{\sigma_x^{(l)}})^e$. Then the partial derivative between $h_x^{(l)}$ and $h_y^{(0)}$ is derived as
\begin{equation}
    \begin{split}
        &\frac{\partial h_x^{(l)}}{\partial h_y^{(0)}} = \frac{\partial h_x^{(l)}}{\partial B_x^{(l)}} \frac{\partial B_x^{(l)}}{\partial f_x^{(l)}}
        \frac{\partial f_x^{(l)}}{\partial h_y^{(0)}}\\
        &=diag(1_{B_x^{(l)}>0} )((\frac{1}{\sigma_x^{(l)}})^e\bm I_d - \frac{e}{d (\sigma_x^{(l)})^{(1+0.5e)}} \cdot \\
        &f_x^{(l)}(f_x^{(l)^T} - \mu_x^{(l)} \bm1^T)) \frac{1}{\sqrt{deg_x}} \bm W^{(l)} \sum_{z\in \mathcal{N}(x)} \frac{1}{\sqrt{deg_z}} \frac{\partial h_z^{(l-1)}}{\partial h_y^{(0)}},
    \end{split}
\end{equation}
where $1_{B_x^{(l)}>0}$ is an indicator vector indicates whether each element of $B_x^{(l)}$ is activated or not. By chain rule, we have
\begin{equation}
    \begin{split}
        &\frac{\partial h_x^{(k)}}{\partial h_y^{(0)}} = \frac{\sqrt{deg_x}}{\sqrt{deg_y}} \sum_{p=1}^\Psi  \prod_{l=k}^1 \frac{1}{deg(v_p^l)} diag(1_{B_{v_p^l}>0}) \Theta_{v_p^l} \bm W^{(l)} ,
    \end{split}
\end{equation}
where $\Psi$ is the number of paths $v_p^k v_p^{k-1},...,v_p^0$ of length k+1 that start from node $x$ and end with $y$, $v_p^{l-1} \in \mathcal{N}(v_p^l)$, $\Theta_{v_p^l}$ represents the Jacobian matrix produced by scale in Eq.(\ref{Eq:jacobian_std}).
Similar to~\cite{DBLP:conf/icml/XuLTSKJ18}, the expectation of $\frac{\partial h_x^{(l)}}{\partial h_y^{(0)}}$ can be estimated as
\begin{equation}
  E[\frac{\partial h_x^{(l)}}{\partial h_y^{(0)}}] = \frac{\sqrt{deg_x}}{\sqrt{deg_y}} \rho \prod_{l=k}^1 \Theta_{v_p^l} \bm W^{(l)} (\sum_{p=1}^\Psi \prod_{l=k}^1 \frac{1}{\deg(v_p^l)}),
\end{equation} 
where $0<\rho<1$ is a probability. Consider when we do not apply the scale, \textit{i}.\textit{e}., $E[\frac{\partial h_x^{(l)}}{\partial h_y^{(0)}}] = \frac{\sqrt{deg_x}}{\sqrt{deg_y}} \rho \prod_{l=k}^1 \bm W^{(l)} (\sum_{p=1}^\Psi \prod_{l=k}^1 \frac{1}{\deg(v_p^l)})$. Note that $(\sum_{p=1}^\Psi \prod_{l=k}^1 \frac{1}{\deg(v_p^l)})$ is exactly the probability that a k-step random walk will stop at $y$ when it starts from $x$, the key point is that $\prod_{l=k}^1 \bm W^{(l)}$ is irrelevant to the starting and ending point, thus the sum of entries of $\prod_{l=k}^1 \bm W^{(l)}$ is the same for all paths. Since multiplied by a same term for all paths does not change the probability (ratio), the expectation of influence distribution is still equivalent to the $k$-step random walk distribution. Although GCN has a bias term $\frac{\sqrt{deg_x}}{\sqrt{deg_y}}$ that penalizes the nodes with large degrees, the influence distribution is still \textit{static} as $\frac{\sqrt{deg_x}}{\sqrt{deg_y}}$ cannot be modified during training. In contrast, the scale introduces an additional derivative $\Theta_{v_p^l}$ that is trainable and varies with the current node representation $f_{v_p^l}$, making the influence distribution dynamic and adaptive. The norm of $\Theta_{v_p^l}$ is mainly determined by $\frac{1}{\sigma_{v_p^l}^e}\bm I$ due to the small norm of $\frac{e}{d \sigma_{v_p^l}^{1+0.5e}} f_{v_p^l}(f_{v_p^l}^T - \mu \bm1^T)$. Hence, the sum of the absolute values of the entries of $\prod_{l=k}^1 \Theta_{v_p^l} \bm W^{(l)}$ largely depends on the magnitudes of $\frac{1}{\sigma_{v_p^l}^e}$. This manifests that if $y$ reaches $x$ via hub nodes (large-std nodes), then the relative influence of $y$ on $x$ will be reduced, which in turn reduces the relative influence of the hub nodes.
\end{proof}

\subsection{Proof of proposition \ref{prop:2}}
\begin{proof}
We first derive $\bm h'$.
\begin{equation}
    \begin{split}
         \bm h' &=  f(\text{Scale}(\bm W' \bm x_i)) = f(\frac{\bm W' \bm x_i \overline{\sigma \prime^\gamma}}{\sigma_i \prime^ {\gamma+e}} ), \\
         &= f(\frac{\delta \bm W \bm x_i  \overline{(\delta \sigma) ^\gamma}}{(\delta \sigma)^{\gamma+e}})
        =\delta ^ {(1-e)} f(\frac{ \bm W \bm x_i  \overline{\sigma^\gamma}}{\sigma_i^{(e+\gamma)}}),\\
        &= \delta ^ {(1-e)} f(\text{Scale}(\bm W \bm x_i))
        = \delta ^ {(1-e)} \bm{h}.   
    \end{split}
\end{equation}
Since $0<e<1$, we further have $|\frac{\bm h'}{\bm h}-1| - | \frac{\bm W'}{\bm W}-1|$ = $|\delta^{(1-e)}-1| - |\delta-1| \le 0$.
\end{proof}

\subsection{Proof of theorem \ref{thm:1}}
\begin{proof}
Let $\lambda_1 \leq \lambda_2 \leq ... \leq \lambda_d$ be the eigenvalues of $\bm {W^T} \bm{W}$ and $\mu_1 \leq \mu_2 \leq ... \leq \mu_d$ be the eigenvalues of $\bm{ W^T S^T S  W}$.

As singular values of a matrix $\bm P$ are equal to the square root of the eigenvalues of $\bm {P^T} \bm P$, we only need to prove $\lambda_1 \ge \mu_1$ and $\lambda_d \ge \mu_d$. Note that $\bm{S^TS}$ is a real symmetric matrix and the eigenvalues of $\bm{S^TS}$ are $0,1,1,...,1$, there exists an orthogonal matrix $\bm{U}$ that satisfies $\bm{S^TS=U^T} \text{diag}(0,1,...,1) \bm{U} $. We begin with
\begin{equation}
\begin{split}
 &\bm{W^TW - W^T S^T S W} = \bm{W^T (I_d - S^T S) W}, \\
 &= \bm{W^T U^T(I_d}-\text{diag}(0,1,...,1)) \bm{U W},\\
 & = \bm{W^T U^T }\text{diag}(1,0,...,0) \bm{UW}, \\
 & = \bm{W^T U^T} (\text{diag}(1,0,...,0))\bm{^T} \text{diag}(1,0,...,0) \bm{UW}, \\
& = \bm{D^TD},
\end{split}
\end{equation}
where $\bm {D= }\text{diag}(1,0,...,0) \bm{UW}$. Let $\bm \xi$ be the eigenvector associated with $\lambda_1$ such that $\bm{\xi^T\xi}=1$ and $\bm{W^TW\xi=} \lambda_1\bm{\xi}$. Then, we have 
\begin{equation}
  \begin{split}
       &0 \leq \bm{\xi^T} \bm{D^TD\xi} =\bm{\xi^T} \bm{W^TW \xi -\xi} \bm{W^TS^T SW \xi},\\
       &=\lambda_1 \bm{\xi^T \xi} - \bm{\xi W^TS^T SW \xi}
      = \lambda_1 - \bm{\xi W^TS^T SW \xi}.
  \end{split}  
\end{equation}
Since $\bm{W^T S^T S W}$ is a real symmetric matrix, $\bm{W^T S^T S W}$ can be decomposed as $\bm{Q^T }\text{diag}(\mu_1,\mu_2,...,\mu_d)\bm{Q}$, where $\bm{Q}$ is an orthogonal matrix. Then $\lambda_1 - \bm{\xi W^TS^T SW \xi} = \lambda_1 - \bm{\xi^T Q^T}\text{diag}(\mu_1,\mu_2,...,\mu_d)\bm{Q\xi}$. We denote $\bm{Q \xi = \pi} = [\pi_1,\pi_2,...,\pi_d]$. Note that $\bm{(Q\xi)^T Q\xi = \xi^T Q^T Q \xi=\xi^T\xi=1}$, $\bm{\pi}$ is also an unit vector such that $\sum \pi_i^2=1$.
We derive $\lambda_1 > \mu_1$ as below.
\begin{equation}
    \begin{split}
        0 \leq  \lambda_1 - &\bm{\xi^T Q^T}\text{diag}(\mu_1,\mu_2,...,\mu_d)\bm{Q\xi} \\
        &= \lambda_1 - \bm{\pi^T}\text{diag}(\mu_1,\mu_2,...,\mu_d)\bm{\pi},\\
        &=\lambda_1 - \mu_1 \pi_1^2+\mu_2 \pi_2^2+...+\mu_d \pi_d^2,\\
        &\leq\lambda_1-\mu_1(\pi_1^2+\pi_2^2+...+\pi_d^2)=\lambda_1 - \mu_1.
    \end{split} \label{eq:10}
\end{equation}
Similarly, let $\bm \gamma$ be the unit eigenvector of  $\mu_d$. Then we can immediately obtain $\bm{\gamma W^TW  \gamma}\leq \lambda_d$ similar to inequality (\ref{eq:10}). Finally, we derive $\lambda_d > \mu_d$ as follows,
\begin{equation}
\begin{split}
    0 &\leq  = \bm{\gamma W^TW  \gamma- \gamma^T W^T S^T S W \gamma}\\
    &= \bm{\gamma W^TW  \gamma}-\mu_d \leq \lambda_d -\mu_d.
\end{split}
\end{equation}
With $\lambda_1\ge\mu_1\ge0$ and $\lambda_d\ge\mu_d\ge0$, we can obtain $\eta_1 = \sqrt{\lambda_1} \ge \sqrt{\mu_1}=\kappa_1$ and $\eta_d = \sqrt{\lambda_d} \ge \sqrt{\mu_d}=\kappa_d$, which completes the proof.

\end{proof}
    Remark 1 can be easily justified in the same way. Similarly, let $\lambda_1$ and $\mu_1$ be the smallest eigenvalue of $\bm{W^T S_i^T S_i W }$ and $\bm{W^T S^T S W}$, respectively.
    \begin{equation}
    \bm {W^T S_i^T S_i W - W^T S^T S W}= \frac{1-k_i}{d}\bm{W^T 1 1^T W}
    \end{equation}
    Since $0 < 1 - k_i < 1$, it follows that $\frac{1-k_i}{d} > 0$. Let $\bm{D} = \sqrt{\frac{1-k_i}{d}}\bm{1^T W}$ and $\bm{\xi}$ be the unit eigenvector associated with $\lambda_1$, such that $\bm{W^T S_i^T S_i W\xi} = \lambda_1\bm{\xi}$ is satisfied. We reach,

\begin{equation}
  \begin{split}
       &0 \leq \bm{\xi^T} \bm{D^TD\xi} =\bm{\xi^T} \bm{W^T S_i^T S_i W \xi -\xi} \bm{W^TS^T SW \xi},\\
       &=\lambda_1 \bm{\xi^T \xi} - \bm{\xi W^TS^T SW \xi}
      = \lambda_1 - \bm{\xi W^TS^T SW \xi}.
  \end{split}  
\end{equation}
The following steps are identical to those in the proof of theorem 1.

\subsection{Dataset}
Here we briefly introduce the datasets used in our experiments.

Cora, Citeseer, and Pubmed are citation networks where nodes represent scientific papers and edges are citation relationships. Node features are bag-of-words representations and each label represents the field that the paper belongs to.

Texas is collected as part of CMU WebKB
project. In this dataset, nodes are university web pages and edges are hyperlinks
between these pages. The task is to classify the nodes into one of the five categories, student, project, course, staff, and faculty.

Actor is a co-occurrence network generated from the film-director-actor-writer network, where node features are bag-of-words representations of the Wikipedia pages of actors. Edges symbolize the two actors' co-occurrence on the same web page. The task is to classify the nodes into five categories in terms of words of the actor’s Wikipedia.

Squirrel is a network of web pages on Wikipedia regarding animals. Node features are bag-of-words representations of nouns on the respective pages. The task is to classify pages into five categories based on the average traffic they receive.

Email is an e-mail network
between members of a European research institution, where each
node is a member, and the edges denote the e-mail communications
between members.

Amazon products are co-purchase networks where nodes represent goods and edges represent that two goods are frequently bought together. Given product reviews as bag-of-words node features, the task is to map goods to their respective product category.

Coauthor CS and Coauthor Physics are co-author networks where nodes represent authors that are connected by an edge if they co-authored a paper. Given paper keywords for each author’s papers, the task is to map authors to their respective field of study.

ogbn-arxiv is a directed graph, representing the citation network between all Computer Science (CS) arXiv papers. Each node is an arXiv paper and each directed edge indicates that one paper cites another one. Each paper comes with a 128-dimensional feature vector obtained by averaging the embeddings of words in its title and abstract. The embeddings of individual words are computed by running the skip-gram model~\cite{mikolov2013distributed} over the MAG corpus.
\end{document}